\documentclass[twoside]{article}

\usepackage{amsmath, amssymb} 
\usepackage[round]{natbib}

\usepackage{authblk}
\usepackage{algorithm}
\usepackage{algorithmic}
\usepackage{xspace}
\usepackage{comment}
\usepackage{xcolor}
\usepackage{booktabs} 
\usepackage[toc,page]{appendix}

\newcommand{\tautarget}{{\alpha}}

\newcommand{\tauprior}{{\tau_\textup{prior}}}

\newcommand{\cprob}{\pi}

\newcommand{\Itwo}{{\mathcal{I}_2}}
\newcommand{\Ione}{{\mathcal{I}_1}}
\newcommand{\Var}{\mathrm{Var}}
\newcommand{\E}{\mathbb{E}}
\newcommand{\Cov}{\mathrm{Cov}}
\newcommand{\1}{\mathbb{I}}
\newcommand{\ttnaive}{\texttt{Naive}\xspace}
\newcommand{\indep}{{ \perp\!\!\!\perp}}

\newcommand{\ttct}{\texttt{Adaptive-T}\xspace}
\newcommand{\model}{\mathcal{G}}
\newcommand{\Xtest}{X_\textup{test}}
\newcommand{\Ttest}{T_\textup{test}}
\usepackage{amsmath,amssymb,amsthm,bbm,color,nccmath}
\usepackage{hyperref}
\usepackage{graphicx}
\usepackage[capitalize,noabbrev]{cleveref}
\usepackage{algorithm}
\usepackage{algorithmic}
\usepackage{caption}
\usepackage{subcaption}
\usepackage[export]{adjustbox}
\usepackage{afterpage}
\usepackage[hang,flushmargin]{footmisc}
\usepackage{float}
\usepackage{enumitem,kantlipsum}
\usepackage{multirow}
\usepackage{array}
\newcommand{\RETURN}{\textbf{return} }

\usepackage{mathtools}

\theoremstyle{plain}
\newtheorem{theorem}{Theorem}[section]
\newtheorem{proposition}[theorem]{Proposition}

\theoremstyle{definition}

\newtheorem{assumption}[theorem]{Assumption}
\theoremstyle{remark}

\usepackage[accepted]{aistats2026}
\usepackage[export]{adjustbox}
\makeatletter
\def\blfootnote{\xdef\@thefnmark{}\@footnotetext}
\makeatother

\usepackage{xr}
\makeatletter
\newcommand*{\addFileDependency}[1]{
  \typeout{(#1)}
  \@addtofilelist{#1}
  \IfFileExists{#1}{}{\typeout{No file #1.}}
}
\makeatother

\begin{document}
\runningtitle{Calibrated Predictive Lower Bounds on Time-to-Unsafe-Sampling in LLMs}
\twocolumn[
\aistatstitle{\texorpdfstring{Calibrated Predictive Lower Bounds \\ on Time-to-Unsafe-Sampling in LLMs}{Calibrated Predictive Lower Bounds on Time-to-Unsafe-Sampling in LLMs}}



\aistatsauthor{ 
  Hen Davidov$^{1,2 *}$ \quad 
  Shai Feldman$^{1 *}$ \quad 
  Gilad Freidkin$^1$ \quad 
  Yaniv Romano$^{1,3}$ 
}

\runningauthor{Hen Davidov,
  Shai Feldman,
  Gilad Freidkin,
  Yaniv Romano}

\aistatsaddress{ 
  $^1$ Department of Computer Science, Technion IIT, Haifa, Israel \\ 
  $^2$ Department of Statistics, University of Oxford, Oxford, United Kingdom \\
  $^3$ Department of Electrical and Computer Engineering, Technion IIT, Haifa, Israel 
}
]
\begin{abstract}
  We introduce time-to-unsafe-sampling, a novel safety measure for generative models, defined as the number of generations required by a large language model (LLM) to trigger an unsafe (e.g., toxic) response. While providing a new dimension for prompt-adaptive safety evaluation, quantifying time-to-unsafe-sampling is challenging: unsafe outputs are often rare in well-aligned models and thus may not be observed under any feasible sampling budget.  
  To address this challenge, we frame this estimation problem as one of survival analysis. We build on recent developments in conformal prediction and propose a novel calibration technique to construct a lower predictive bound (LPB) on the time-to-unsafe-sampling of a given prompt with rigorous coverage guarantees.
  Our key technical innovation is an optimized sampling-budget allocation scheme that improves sample efficiency while maintaining distribution-free guarantees.
  Experiments on both synthetic and real data support our theoretical results and demonstrate the practical utility of our method for safety risk assessment in generative AI models.
\end{abstract}

\section{Introduction}
Despite substantial progress in aligning LLMs to safety guidelines, current systems remain vulnerable to rare but impactful unsafe outputs that emerge through repeated sampling \citep{bianchi2024safety, Vidgen2023SimpleSafetyTestsAT, RAY2023121}. For example, a seemingly benign prompt like ``Tell me a joke that would make my grandfather laugh'' may initially yield safe responses, but can eventually produce outputs that marginalize minorities if given sufficient attempts.

We define an unsafe response as content that includes toxic language \citep{Vidgen2023SimpleSafetyTestsAT}, the disclosure of private or sensitive information \citep{kaddour2023challenges}, or deepfakes that defame individuals \citep{westerlund2019emergence}. A standard \emph{retrospective} approach to protect users from such outputs is to audit the LLM’s response at inference time using an \texttt{Audit} function \citep{DBLP:journals/corr/abs-2312-06674, hutoxicity}. This mechanism can either trigger refusal or allow the model to generate new responses until a safe one is observed \citep{ahmed2024controllable, bhatt2025ctrl, stroebl2024inference}.

Our work belongs to a family of LLM safety methods that adopt a \emph{proactive} approach: rather than waiting for a prompt $X$ to elicit an unsafe response, we aim to predict the risk associated with the prompt in advance. Such safety methods estimate the probability of an unsafe event $\mathbb{P}(Y = 1 \mid X)$, where $Y$ is a binary safety indicator \citep{perez2022red, su2024mission, xu2025utilizing}. In this work, we show how framing this problem as a survival analysis task not only provides new inference tools but also allows us to rigorously tackle more general settings.

To understand the challenges, one can view the unsafe probability estimation as a process that independently generates multiple LLM responses for a given $X$, audits their safety, and computes the average of the resulting unsafe indicators. This process highlights two key limitations. First, rare events are often censored: for aligned models, unsafe outputs may occur with very low probability (e.g., $10^{-6}$) and thus remain unobserved under reasonable sampling budgets $B \in \mathbb{N}$ (e.g., $10$ LLM generations per prompt). In such cases, naively fitting a classifier to estimate the unsafe probability can misleadingly ignore the underlying risk. Second, LLM generations are not always independent. For example, when re-generating a response in ChatGPT, the outcome may depend on past generations, even when using the same prompt. Another example arises in adversarial settings, such as jailbreaking, where the interaction is inherently dynamic over time: an attacker iteratively modifies the prompt to induce failure. In these cases, the probability of an unsafe event is time-dependent; consequently, standard estimators become invalid \citep{rao2023tricking, dickson2024estimating}.

To handle these limitations in a unified and rigorous way, we propose a fundamental change in perspective: rather than asking what is the probability that the model will fail, we ask \emph{when the first failure will occur}. We introduce \textbf{time-to-unsafe-sampling}, defined as the number of sampling attempts required to elicit an unsafe response for a given prompt $X$. 

In the simplest setting where the prompts and responses are sampled i.i.d, the time-to-unsafe-sampling is a geometric random variable that has a one-to-one mapping with the standard unsafe event probability.  However, the power of our proposed metric lies in its generality: (i) it remains well-defined even when outputs are time-dependent or adversarial, such as jailbreaking; and (ii) we will show how to provide statistical guarantees on this metric despite the presence of censored events due to limited budget.
Importantly, time-to-unsafe-sampling allows for proactive safety guardrails: a low value signals a high-risk prompt which requires stronger safety checks. Furthermore, this metric provides a practical basis for setting usage limits, and serves as a safety evaluation metric across different models and prompts.

Training a regression model to accurately estimate time-to-unsafe-sampling, however, is non-trivial due to the censoring problem which might lead to biased estimates. 
Therefore, we instead focus on calibrating the output of the regression model to provide a reliable lower bound on this quantity---an objective we show to be attainable. Denote by $\hat{L}(\Xtest)$ a lower predictive bound (LPB) on the unknown time-to-unsafe-sampling $\Ttest$ of a new test prompt $\Xtest$. We define the coverage rate of $\hat{L}(\Xtest)$ as the probability that the true (unknown) time-to-unsafe-sampling $T_\textup{test}$ exceeds the LPB, i.e., $\mathbb{P}(T_\textup{test}\geq \hat{L}(X_\textup{test}))$. Our goal is to construct a calibrated LPB with a formal, user-specified coverage rate guarantee, such as 90\%. Informally, this means that, with high probability, one can expect to generate at least $\hat{L}(\Xtest)$ safe responses to the test prompt $X_\textup{test}$ before encountering an unsafe one across future test prompts.

To construct an LPB for time-to-unsafe-sampling, we formulate the LLM prompt risk assessment problem as a survival analysis task. In survival analysis, the goal is to predict a time-to-event---e.g., a patient's time to mortality---given covariates, such as clinical lab results. Similar to our setting, an inherent challenge in survival data is that the time-to-event can be censored due to the patient's withdrawal of consent or the end of the clinical trial~\citep{machin2006survival}. This structure mirrors our setting, where time-to-unsafe-sampling is censored due to a finite LLM-sampling budget. Leveraging this analogy, we adopt recent conformalized survival analysis methods, which provide a principled way to construct LPBs with rigorous, finite-sample coverage guarantees under censoring~\citep{gui2024conformalized,candes2023conformalized}.


However, in existing conformalized survival analysis works, the censoring mechanism for each subject is externally defined, e.g., by study design, and treated as fixed. Put simply, these methods do not intervene in how the data is collected. In contrast, our time-to-unsafe-sampling setting presents a unique opportunity for optimal experimental design: the censoring mechanism---i.e., how many LLM samplings are allocated per prompt---can be \textbf{actively designed} under a sampling budget $B$. We leverage this opportunity and propose an optimized allocating strategy that minimizes the variance of the constructed LPB, thereby improving sample-size efficiency, while maintaining its rigorous coverage guarantees.

In sum, our main contributions are as follows:
\begin{enumerate}[label=(\arabic*)]
  \item We propose time-to-unsafe-sampling: a new safety measure that quantifies the risk of temporal, probabilistic failures at a fine-grained, per-prompt level (Section~\ref{sec:setup}).

  \item We refine the method of \citet{gui2024conformalized}, simplifying it to construct LPBs with a formal coverage guarantee under a known censoring mechanism (Section~\ref{sec:calibration}). Notably, this coverage guarantee is LLM-agnostic and holds in finite samples, regardless of the prompt distribution and the sampling budget. Moreover, obtaining an LPB from our method at inference time requires only a single query to the calibrated regression model, without invoking the \texttt{Audit} model or generating multiple LLM samples.
  \item We analyze how the sampling-budget allocation influences the variance and coverage rate achieved by our LPBs (Section~\ref{sec:methods}). Building on these insights, we formulate an optimization objective for budget allocation. Solving this objective yields a censoring mechanism that offers tighter control over coverage compared to non-optimized baselines (Section~\ref{subsec:prompt_global}).
  \item We validate our theory on simulated and real-data experiments. Specifically, our experiments with the ``RealToxicityPrompts'' dataset \citep{gehman2020realtoxicityprompts} using Llama~3.2 \citep{meta2024llama3_2, grattafiori2024llama3}, demonstrate that the proposed method yields more informative bounds and more reliable risk estimates than baseline approaches (Section~\ref{sec:exp}). Code is provided for reproducibility at~\href{http://github.com/giladfrid009/llm-survival/}{http://github.com/giladfrid009/llm-survival/}.
\end{enumerate}


\section{Problem Setup}
\label{sec:setup}
Let \(\mathcal{D}\!=\!\{X_i\}_{i=1}^n\) be a dataset of user prompts, drawn i.i.d.\ from some distribution \(P_X\). Let $\model(x)$ denote an LLM.  For each prompt \(X_i\), let    
\(
\{\model^j(X_i)\}_{j=1}^\infty
\)
be a sequence of responses generated by repeated calls to \(\model(X_i)\). We define the binary indicator
\(
Y_i^j \!=\! \texttt{Audit}\bigl(X_i,\model^j(X_i)\bigr),
\)
where \(Y_i^j = 1\) if and only if the \(j\)-th response is labeled unsafe by the \texttt{Audit} function. The uncensored \emph{time-to-unsafe-sampling} for prompt \(X_i\) is then
\[
T_i \!=\!\min\bigl\{\,j \ge 1: Y_i^j = 1\bigr\}.
\]
In the simplest case, conditional on $X_i$, the sequence $\{Y_i^j\}$ is i.i.d.\ Bernoulli with unsafe sampling probability $p(X_i):=\mathbb{P}(Y_i^j =1 \mid X_i) $, implying $T_i \mid X_i \sim \mathrm{Geom}\!\bigl(p(X_i)\bigr)$. In practice, however, successive generations may be dependent; for instance, through iterative self-reflective prompting~\citep{madaan2023self, shinn2023reflexion}, jailbreak attacks~\citep{chao2025jailbreaking}, or LLM tool use where tool availability or behavior changes over time~\citep{wangvoyager}. Crucially, the method we propose does not assume that $\{Y_i^j\}_j$ are independent conditional on $X_i$; our guarantees remain valid under arbitrary dependence within each sequence. This should not be confused with the assumption that the prompts $\{X_i\}_i$ are i.i.d., which is required for our theoretical analysis.

Given a user-specified coverage level $1-\alpha \in (0,1)$, and a tolerance level $\delta\in(0,1)$, our goal is to construct an LPB $\hat{L}$ that satisfies
\begin{equation}\label{eq:pac_LPB}
    \mathbb{P}\big(T_\textup{test}\geq \hat{L}(X_\textup{test}) \mid \mathcal{D}\big)\geq 1-\alpha,
\end{equation}
with probability at least $1-\delta$ over the randomness of the training data $\mathcal{D}$. 
The probability in \eqref{eq:pac_LPB} is taken over the distribution of a test point $(X_\textup{test},T_\textup{test})\sim P_{X,T}$. An LPB $\hat{L}(\Xtest)$ satisfying this requirement is called a Probably Approximately Correct (PAC) LPB at level $\alpha$ with tolerance $\delta$.

We follow standard split-conformal methodology \citep{vovk2005algorithmic}, and partition the index set \(\{1,\dots,n\}\) of the training prompts $\mathcal{D}$ into a \emph{proper training set} \(\mathcal{I}_{1}\) and a holdout \emph{calibration set} \(\mathcal{I}_{2}\). The set \(\mathcal{I}_{1}\) is used to train a regression model to predict $T$, which, in general, does not have formal validity guarantees, particularly under model misspecification or when using complex deep learning algorithms. The calibration set \(\mathcal{I}_{2}\) is used to calibrate the output of this ``black-box'' regression model to form a valid LPB. Our focus in this work is on the calibration stage.


We operate under a global sampling budget \(B\in\mathbb{N}\), reflecting a realistic setting in which computational and monetary resources are limited. To respect this budget during calibration, we introduce a per-prompt \textbf{censoring time} \(C_i\in\mathbb{N}\), which limits the number of response–audit rounds performed for each prompt \(X_i\). Formally, we require that
\begin{equation}
    \label{eq:budget_constraint}
\mathbb{E}\Bigl[\sum_{i\in\Itwo} C_i
\;\Big|\;
\{X_i\}_{i\in\Itwo}\Bigr]
\!\le\! B.
\end{equation}
In words, the above ensures that the expected number of response–audit rounds does not exceed the overall budget $B$.\footnote{We remark that in our setup, the training procedure has a separate sampling budget.}

In Section~\ref{sec:methods}, we propose specific schemes for designing the censoring mechanism under the budget constraint \eqref{eq:budget_constraint}. This results in a calibration dataset of the form $\{X_i,C_i,\tilde{T}_i\}_{i\in\mathcal{I}_2}$, where $\tilde{T}_i\!=\!\min\bigl(T_i,C_i\bigr)$ denotes the censored time-to-unsafe-sampling. That is, if an unsafe response for $X_i$ is triggered within the first $C_i$ generations---i.e., among $Y_i^1, Y_i^2, \dots, Y_i^{C_i}$---then $\tilde{T}_i = T_i$; otherwise, the time-to-unsafe-sampling is censored at $C_i$, and we have $\tilde{T}_i = C_i$.

In addition to satisfying the global budget constraint~\eqref{eq:budget_constraint}, we design the censoring mechanism to 
satisfy a widely adopted assumption in the survival literature that the censoring time $C$ is conditionally independent of the survival time $T$ given the covariates $X$, which provides a theoretical guarantee on the LPB we construct:
\begin{assumption}[Conditionally Independent Censoring \citep{kalbfleisch2011statistical}]    \label{assumption:conditionally_independent_censoring}
    $C \  \indep \ T \mid X.$
\end{assumption}
This assumption means we cannot dynamically adjust the sampling budget based on the generated LLM responses, as doing so would induce dependencies between $T_i$ and $C_i$, invalidating the statistical calibration. However, it allows us to optimize the allocation $C_i$ based on the prompt $X_i$ before sampling begins--a key degree of freedom we exploit in Section~\ref{subsec:prompt_global}.




\section{Related Work}
\label{sec:related}

In this section, we present the conformalized survival analysis algorithm of~\citet{gui2024conformalized}, called \ttct, which constructs a valid LPB in the sense of \eqref{eq:pac_LPB} under Assumption~\ref{assumption:conditionally_independent_censoring}. Additional related work is discussed in Appendix~\ref{sec:additional_related_work}.

Given calibration data $\{X_i,C_i, \tilde{T}_i\}_{i\in\mathcal{I}_2}$, \ttct constructs a PAC-type LPB $\hat{L}$ by calibrating a pre-trained quantile regression model. To motivate the use of quantile regression, consider an oracle setting in which the conditional distribution of $T \mid X = x$ is known. Let $q_\tau(x)$ denote the true $\tau$-th conditional quantile of $T \mid X = x$. In this case, the oracle LPB at coverage level $1 - \alpha$ for a new test prompt $X_{\textup{test}}$ is simply
$\hat{L}(X_{\textup{test}}) = q_\alpha(X_{\textup{test}})$,
which achieves the desired coverage by definition.

Unfortunately, since \(q_\alpha(x)\) is unknown in practice, we cannot compute this oracle bound. Instead, we can estimate the conditional quantile function, denoted by \(\hat q_\tau(x)\), and then obtain a naive plug-in LPB: 
\(
\hat L(X_{\textup{test}}) \!=\! \hat q_\alpha(X_{\textup{test}}).
\)  
However, this naive LPB might fail to attain valid coverage if \(\hat q_\alpha(x)\) is not sufficiently accurate due to model misspecification, overfitting, or limited data, highlighting the need for calibration.

The \ttct method addresses this limitation by finding a calibrated  quantile level $\hat{\tau}^{\text{adapt}}$ such that \(
\hat{L}(X_{\textup{test}})\!=\! \hat{q}_{{\hat{\tau}}^{\text{adapt}}}(X_{\textup{test}})\) approximately satisfies \eqref{eq:pac_LPB} in a PAC sense. At a high level, this calibration is done by estimating the miscoverage probability  
\(
\mathbb{P}\bigl(T < \hat q_\tau(X)\bigr)
\) 
for each candidate \(\tau\) using the held-out calibration set $\Itwo$. Then, this method proceeds by choosing the largest \(\tau\) for which all $\tau' < \tau$ have estimated miscoverage less than or equal to \(\alpha\). 

In detail, for each level $\tau$, the miscoverage estimator $\hat{\alpha}^{\text{adapt}}(\tau)$ proposed by \citet{gui2024conformalized} is given by
\begin{equation}
\label{eq:gui_miscoverage}    
\hat\alpha^{\text{adapt}}(\tau)
=
\frac{
  \sum_{i\in\mathcal I_2}\hat w_\tau(X_i)\,\mathbb I\{ \tilde{T}_i<\hat{q}_\tau(X_i)\le C_i\}
}{
  \sum_{i\in\mathcal I_2}\hat w_\tau(X_i)\,\mathbb I\{\hat q_\tau(X_i)\le C_i\}
}.
\end{equation}
We pause to clarify the above expression. First, observe that the indicator $\mathbb I\{ \tilde{T}_i<\hat{q}_\tau(X_i)\le C_i\}$ is a function of observed quantities. Second, under Assumption~\ref{assumption:conditionally_independent_censoring}, \citet{gui2024conformalized} show that considering only samples for which $\hat{q}_\tau(X_i)\le C_i$ introduces a covariate shift. To account for this shift, each selected example is re-weighted by an ``inverse-censoring'' weight $\hat{w}_\tau(X_i)$ that approximates $1/\mathbb P\bigl[\hat q_\tau(X_i)\le C_i \bigm| X_i \bigr]$. Lastly, the denominator in \eqref{eq:gui_miscoverage} is used to normalize the weighted average, which is crucial when the estimated weights $\hat{w}_\tau(x)$ are merely proportional (rather than exactly equal) to the true inverse probabilities.

Finally, \ttct defines the calibrated quantile level as  
$\hat L(x)=\hat q_{\hat\tau^{\text{adapt}}}(x)$,
where
\begin{equation}
\label{eq:gui_tune_tau}
\hat\tau^{\text{adapt}}
=\sup\Bigl\{\tau\in\mathcal T : \sup_{\tau'<\tau}\hat\alpha^{\text{adapt}}(\tau')\le\alpha\Bigr\},
\end{equation}
where $\mathcal{T}$ is the search space for $\tau$.
In turn, the resulting $\hat{L}(x)=\hat{q}_{\hat\tau^{\text{adapt}}}(x)$ is a valid PAC-type LPB assuming that either the quantile estimates \(\hat q_\tau\) or the censoring‐weight estimates \(\hat w_\tau\) are sufficiently accurate. 

\section[First steps]{Calibration with a known censoring mechanism}
\label{sec:calibration}



Recall that in our LLM prompt risk assessment setting, we design the sampling-budget allocation scheme. As a result, we have full knowledge of the inverse-censoring weights required to implement the method of~\citet{gui2024conformalized} presented in the previous section. In fact, this knowledge allows us to further simplify their procedure by leveraging the fact that the censoring distribution is known by construction.

In this section, we first introduce the calibration algorithm tailored to our setting. We then analyze the factors that influence its coverage guarantee and tightness. This analysis not only underpins our calibration procedure, but also guides our sampling-budget allocation strategies developed in Section~\ref{sec:methods}.


Formally, the inverse-censoring weight corresponding to each $\hat q_\tau$ is defined as
\begin{equation}\label{eq:weight_def}
    w_\tau(X_i)=1 / \mathbb P\bigl[\hat q_\tau(X_i)\le C_i\bigm| X_i\bigr].
\end{equation}
Using these weights, we estimate the empirical miscoverage rate at quantile level $\tau$ as a weighted average over the calibration set:
\begin{align}
\label{eq:miscoverage_estimator}
    \hat\alpha(\tau)
\;=\;
\frac{1}{|\Itwo|}\sum_{i\in\mathcal I_2}
w_\tau(X_i)\;
\mathbb{I}\bigl\{\tilde{T}_i<\hat q_\tau(X_i)\le C_i\bigr\}.
\end{align}
Note that this estimator differs from the one used in~\eqref{eq:gui_miscoverage}: it does not require additional normalization since our proposed censoring mechanisms provide the true weights $w_\tau(X_i)$, as explained in Section~\ref{sec:methods}.

Armed with $\hat\alpha(\tau)$, we obtain calibrated LPB by following \eqref{eq:gui_tune_tau} and selecting the largest \(\tau\) such that \(\hat\alpha(\tau)\le\alpha\):
\begin{align}
\label{eq:our_calibration}
\hat L(x)=\hat q_{\hat\tau}(x),
\ \
\hat\tau
=\sup\bigl\{\tau \in \mathcal{T}:\sup_{\tau'<\tau}\hat\alpha(\tau')\le\alpha\bigr\}, \ \ \ 
\end{align}
where $\mathcal{T} = \bigl\{\sup_{\tau \in [0,1]} \{\hat{q}_\tau(X_i) \leq \tilde{T}_i\} : i\in\Itwo\} \cup \{\sup_{\tau \in [0,1]} \{\hat{q}_\tau(X_i) \leq C_i\} : {i\in\Itwo}\bigr\} \cup \{0\}$.


This procedure is summarized in Algorithm~\ref{alg:calibration_with_known_censoring}.
The following informal theorem shows that the LPB $\hat{L}$ formulated in \eqref{eq:our_calibration} holds a PAC-type coverage guarantee. In Appendix~\ref{sec:validity_proof} we state the formal version of this theorem and prove it by building on the theoretical framework developed by~\citet{gui2024conformalized}. 

\begin{theorem}[General validity, informal]
\label{thm:general_validity}
Fix a tolerance level $\delta \in \left(0,1\right)$ and a miscoverage level $\tau\in \left(0,1\right)$. Suppose that $\{(X_i, T_i)\}_{i=1}^{n}$ and $(\Xtest, \Ttest)$ are drawn i.i.d and that the censoring times satisfy the conditional independence assumption (Assumption~\ref{assumption:conditionally_independent_censoring}). Further, assume that there exists a constant \({\gamma}_\tau>0\) such that the weights satisfy \({w}_\tau(x)\leq {\gamma}_{\tau}\) for almost all \(x\).
Then, with probability at least $1-\delta$ over the draws of $\mathcal{D}$, the LPB $\hat{L}(x) = \hat{q}_{\hat{\tau}}(x)$ from \eqref{eq:our_calibration} satisfies
\begin{align*}
\mathbb{P}\left[\Ttest\geq \hat{L}(\Xtest)|\mathcal{D}\right] &\geq 1 - \alpha -\\
&\sup_{\tau \in [0,1]} \left\{ \sqrt{ \frac{ 2\gamma^2_{\tau} + 5  }{|\Itwo|}  \cdot  \log \left(\frac{1}{\delta} \right) }\right\}.
\end{align*}
\end{theorem}
Importantly, the above guarantee holds in finite samples, for any LLM $\mathcal{G}(x)$, any prompts distribution \(P_X\), sampling-budget $B$, and regardless of the accuracy of the quantile estimators $\hat{q}_\tau(x)$. However, this coverage bound depends on the supremum of all possible weights with non-zero probability, denoted by $\gamma_\tau$. In particular, as $\gamma_\tau$ increases, the bound becomes looser.

In Appendix~\ref{appendix:alpha_hat}, we further analyze the calibration method in greater depth by studying the properties of $\hat\alpha(\tau)$ defined in \eqref{eq:miscoverage_estimator}. Specifically, in Proposition \ref{prop:variance_constant_miscoverage}, we make an illustrative assumption that, for any fixed $\tau$, each time-to-unsafe-sampling $T_i$ is miscovered by $\hat q_\tau(X_i)$ at the same constant rate, conditional on the calibration prompts $\{X_j\}_{j\in\mathcal I_2}$. This assumption is satisfied by the oracle conditional quantile $q_\tau(\cdot)$ because, by definition, $\mathbb{P}\left[T_i<q_\tau(X_i)\mid X_i\right]=\tau$ for every $i$, and hence also conditional on $\{X_j\}_{j\in\mathcal I_2}$. Under this assumption, we prove that the conditional variance
$
\Var\!\bigl[\hat\alpha(\tau)\,\bigm|\,\{X_j\}_{j\in\mathcal I_2}\bigr]
$
is \textbf{linearly monotone} with regard to the mean calibration weight 
\begin{equation}
\overline w_\tau=\frac1{|\Itwo|}\sum_{i\in\Itwo}w_\tau(X_i).    
\end{equation}
The above discussion shows that any censoring design producing large weights will both weaken the PAC‐type coverage bound (through a large supremum weight $\gamma_\tau$) and increase the variance of the empirical miscoverage estimator (through a large mean weight $\bar w_\tau$). Indeed, these two properties guide our \texttt{Optimized} censoring design presented in the following section, where we both (i) cap $\gamma_\tau$ by a constant value; and (ii) minimize the mean weight $\bar w_\tau$ subject to fully utilizing the sampling-allocation budget~$B$. 

\section{Proposed Budget-Allocation Mechanism Designs}\label{sec:methods}

Having formulated the general calibration scheme and the factors governing its coverage, we first introduce a \ttnaive budget allocation to illustrate the inherent challenges of this task. Then, we present our flagship \texttt{Optimized} approach, which effectively overcomes the limitations of this naive baseline.


\subsection{\texttt{Naive} Budget Allocation}\label{sec:naive_allocation}

As a conceptual warm-up, we introduce a simple baseline budget allocation strategy, the \ttnaive approach, which defines each censoring time $C_i$ as a random variable \emph{independent of the prompt} $X_i$. Under the naive assumption that the LLM outputs are i.i.d, $T_i | X_i$ is a Geometric random variable with infinite support. Hence, we design $C_i$ to be distributed geometrically as well: 
$C_i \sim \mathrm{Geom}(\min(|\Itwo| / B, 1)), \ \forall i \in \Itwo$,
which satisfies the budget constraint in~\eqref{eq:budget_constraint} by design. This \ttnaive approach is summarized in Algorithm~\ref{alg:calibration_naive} in Appendix~\ref{sec:calibration_naive}.

Under this \texttt{Naive} allocation, the event $\hat{q}_\tau(X_i) \le C_i$ can be rare, especially for large $\hat{q}_\tau(X_i)$. This makes $\bar w_\tau$ and $\gamma_\tau$ excessively large. Consequently, the resulting coverage rates might deviate from the nominal level---as reflected in our experiments in Section~\ref{sec:exp}. This discussion motivates the method presented below, which utilizes the budget more efficiently and designs the censoring time as a function of the prompt.

\subsection[\texttt{Optimized} mode]{Our Main Proposal: \texttt{Optimized} Allocation}
\label{subsec:prompt_global}

To better utilize the calibration budget in the computation of $\hat{\alpha}(\tau)$ in~\eqref{eq:miscoverage_estimator}, we aim to design prompt-dependent censoring times $C_i$ that maximize $\mathbb{P}\bigl(\hat q_{\tau}(X_i)\le C_i\bigr)$.
Ideally, we would like to set the censoring times for each $\tau$ to increase this probability.
However, this is infeasible as we construct the dataset only once. As a result, we must commit to a single budget allocation by choosing a specific quantile level $\tau$ to consider.
For this purpose, we assume that the quantile estimators are monotone in $\tau$. This assumption can be enforced by sorting $\hat{q}_\tau(X_i)$. Consequently, $
\forall \tau_1 \le \tau_2, \ 
\mathbb{I}\{\hat q_{\tau_2}(X_i) \le C_i\}
\Rightarrow
\mathbb{I}\{\hat q_{\tau_1}(X_i) \le C_i\}.
$
Thus, if we were to naively set $C_i=\hat{q}_{1}(X_i)$ as the 100\%–quantile, we would indeed maximize $\mathbb{P}\bigl(\hat q_{\tau}(X_i)\le C_i\bigr)$ for all $\tau \in [0,1]$. However, $\hat{q}_{1}(X_i)$ can be excessively large, e.g., this 100\%-quantile is infinite if $T_i \mid X_i$ follows a geometric distribution. As a result, this naive choice is infeasible under our finite budget constraint $\mathbb{E}\left[\sum_{i \in \Itwo} C_i | \{X_i\}_{i \in \Itwo} \right] \leq B$. 

Instead, we consider a more practical compromise: we select a quantile level $\tau_{\text{prior}} \ll 1$ that represents our prior belief $\hat{\tau} \in [0, \tau_{\text{prior}}]$.
For example, if the target coverage level $1-\alpha$ is not too extreme (e.g., $90\%$) and the uncalibrated quantile estimator $\hat{q}_\tau$ is reasonably accurate, a prior level such as $\tau_{\text{prior}}=20\%$ would be a reasonable choice.
Importantly, constraining $\tau$ to lie within this prior range affects only the informativeness of the LPB, and not its validity; see Appendix~\ref{sec:validity_proof}.

After choosing the prior level $\tau_{\textup {prior}}$, we turn to set a prompt-adaptive censoring time $C_i$ using the estimated quantile $\hat q_{\tau_{\textup {prior}}}(X_i)$. Following \eqref{eq:miscoverage_estimator}, the $i$-th sample contributes to $\hat\alpha(\tau)$ if $C_i\leq \hat q_{\tau_{\textup {prior}}}(X_i)$ for all $\tau\in[0,\tau_{\textup {prior}}]$. Combining this observation with the monotonicity of $\hat q_\tau$, a sensible choice for $C_i$ that avoids wasted samplings would either be (i) $C_i=\hat{q}_{ \tau_{\text{prior}}}(X_i)$, or (ii) $C_i=0$. Other choices for $C_i$ that are strictly greater or smaller than $\hat{q}_{ \tau_{\text{prior}}}(X_i)$ would result in unnecessary generations that do not contribute to the miscoverage estimator $\hat\alpha(\tau)$. Furthermore, to bound $\gamma$, we replace the raw quantile estimates $\hat{q}_{ \tau_{\text{prior}}}(X_i)$ with a capped variant $\hat{f}_{\tau}(x) := \min (\hat{q}_{\tau}(x), M)$, where $M$ is a constant defined by the user.
Following this, we define
\begin{equation}\label{eq:c_with_pi}
    C_i := \text{Ber}(\cprob_i) \cdot \hat{f}_{ \tau_{\text{prior}}}(X_i),
\end{equation}
where $\text{Ber}(\cprob_i)$ is a Bernoulli random variable with success probability of our choice $\cprob_i$. 
We now show how to set this probability to minimize the average weight $\bar w=\frac{1}{|\Itwo|}\sum_{i\in\mathcal I_2}\frac1{\pi_i}$ in order to reduce the variance of $\hat{\alpha}(\tau)$. By this design, we get
\(
\mathbb E[\sum_{i\in\mathcal I_2}C_i | \{X_i\}_{i\in\Itwo}]
=\sum_{i\in\mathcal I_2}\hat{f}_{\tauprior}(X_i)\,\pi_i
\). In turn, we can minimize $\bar w$ subject to the budget constraint by solving the following convex optimization problem:
\begin{equation}\label{eq:global‐opt}
\pi^*
=\underset{{\pi\in[0,1]^{|\Itwo|}}}{\rm{argmin}}
\;\frac{1}{|\Itwo|}\sum_{i\in\mathcal I_2}\frac1{\pi_i}
\ \ \text{s.t.} \ \
\sum_{i\in\mathcal I_2}\hat{f}_{\tauprior}(X_i)\,\pi_i \!\le\! B.
\end{equation}
In Appendix~\ref{sec:calibration_algs}, we introduce Algorithm~\ref{alg:SolveOptimization}, which efficiently solves the above optimization problem. Proposition~\ref{prop:unique_lambda} ensures this algorithm returns a unique and strictly positive solution $\pi^*$ that satisfies the budget constraint \eqref{eq:budget_constraint}. Importantly, because the solution $\pi^*$ depends on the entire set of calibration prompts $\{X_i\}_{i\in\Itwo}$, the resulting weights inherit this dependence. To make this explicit, we denote the weight associated with the $i$-th example as $w(\{X_j\}_{j\in\mathcal{I}_2}, i) = 1/\pi_i^*$. We remark that our coverage guarantee from Theorem~\ref{thm:general_validity} holds also for this formulation of the weights; see Appendix~\ref{sec:validity_proof} for details. Furthermore, these weights are also guaranteed to be upper-bounded by the same expression for $\gamma$ by trimming the quantile estimates with $M$:
\begin{proposition}[Maximal weight bound]\label{prop:min_prob_bound}
Suppose that $\max_{i \in \Itwo} \hat{f}_{\tauprior}(X_i) \leq M$, then, the weights \\$w(\{X_j\}_{j\in\Itwo},i)$ induced by solving \eqref{eq:global‐opt} are upper bounded by 
$\gamma = \max \left ({|\Itwo|\cdot M}/{B} , 1 \right).$
\end{proposition}
All proofs are deferred to Appendix~\ref{sec:proofs}. The above proposition shows us that weights derived by the solution $\pi^*$ are bounded by $\gamma$. Once all $C_i$-s are set using the ideal $\pi_i^*$, we can use our calibration procedure from Section~\ref{sec:calibration}, with the sole modification that the search space for $\tau$ is now $\mathcal{T} \cap[0,\tauprior]$.
This \texttt{Optimized} calibration scheme is outlined in Algorithm~\ref{alg:calibration_optimized} in Appendix~\ref{sec:optimized_alg}.

\section{Experiments}
\label{sec:exp}

We evaluate our methods on both synthetic (Section \ref{subsec:synth}) and real (Section \ref{subsec:real}) datasets. For a full description of how we train models to estimate the conditional quantiles of $T$ given $X$, along with additional experiments, implementation details, and runtime considerations for our calibration methods, see Appendix~\ref{sec:details}. 
Since there are no established baseline methods in this novel setting, we hereafter introduce two new calibration baselines, \texttt{Basic} and \texttt{Trimmed}, for comparison. We evaluate our proposed \texttt{Optimized} method against these two new baselines, as well as a vanilla uncalibrated quantile regression model and the calibrated \ttnaive.
The following experiments indicate that our \texttt{Optimized} budget allocation leads to more informative LPBs that attain a coverage rate closer to the desired level compared to the baselines.

\subsection{Baseline Allocation Strategies}

\paragraph{\texttt{Basic}.}\label{subsec:prompt_adaptive}
We follow the censoring times design from the \texttt{Optimized} method, and set them as in \eqref{eq:c_with_pi}:
$C_i := \text{Ber}(\cprob_i) \cdot \hat{q}_{ \tau_{\text{prior}}}(X_i)$.
However, unlike the \texttt{Optimized} approach, here we determine the success probability using a fixed heuristic: $\cprob_i = \min({B}/({|\Itwo| \cdot \hat{q}_{\tau_\text{prior}}(X_i)}), 1).$
We then apply the calibration procedure from Section~\ref{sec:calibration} using these generated censoring times. For convenience, Algorithm~\ref{alg:calibration_basic} in Appendix~\ref{sec:calibration_basic} outlines this \texttt{Basic} budget-allocation procedure.

\paragraph{\texttt{Trimmed}.} \label{subsec:prompt_capped} We extend the \texttt{Basic} budget allocation by trimming the maximal estimated quantile at a fixed threshold $M\in\mathbb N$. We do so similarly to the \texttt{Optimized} approach, to bound $\gamma$.
Consequently, the censoring times are now given by $ C_i := \text{Ber}(\cprob_i) \cdot \hat{f}_{ \tau_{\text{prior}}}(X_i),$ where 
$\hat{f}_{\tau}(x) := \min (\hat{q}_{\tau}(x), M)$ and
$\pi_i = \min({B}/({|\Itwo| \cdot \hat{f}_{\tau_\text{prior}}(X_i)}), 1).$ Following this trimming, the maximum weight $w(x)$ is bounded by $\gamma=\max\left({|\Itwo| \cdot M}/{B}, 1\right)$. As observed in Theorem~\ref{thm:general_validity}, this allows us to obtain a coverage guarantee that is closer to the desired $1-\alpha$ level. 
The rest of the calibration process continues as outlined in Section~\ref{sec:calibration}, using the trimmed estimates $\hat{f}_{\tau}(X_i)$ instead of $\hat{q}_{\tau}(X_i)$.
We summarize this scheme in Algorithm~\ref{alg:calibration_trimmed} in Appendix~\ref{sec:calibration_trimmed}.


\subsection{Synthetic Data Experiments}
\label{subsec:synth}
\begin{figure*}[ht]
    \centering
    \includegraphics[width=\textwidth]{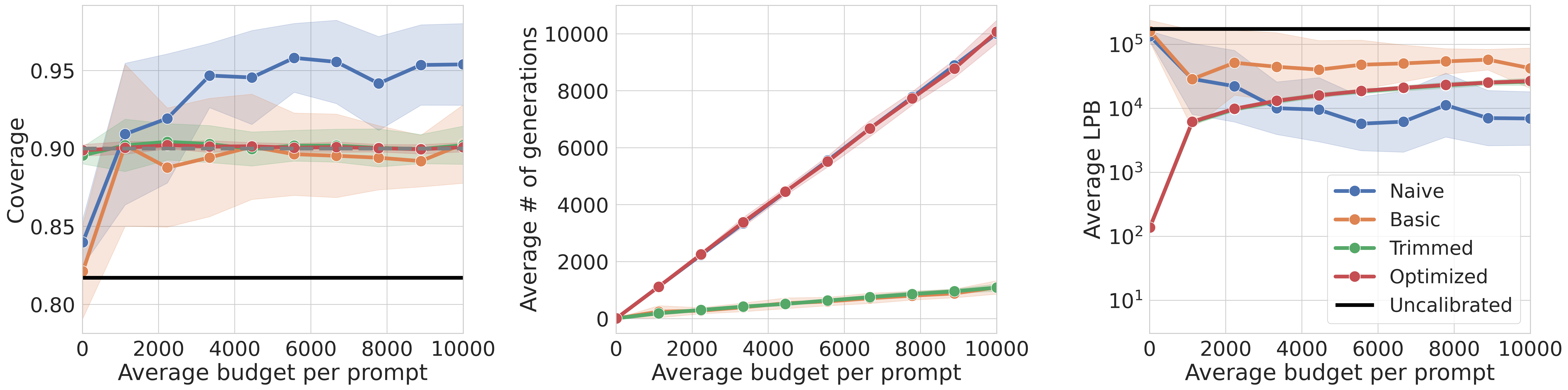}
    \caption{Synthetic experiments. 
    {\bf Left:} Coverage (target 90\%). 
    {\bf Center:} Mean number of samplings generated per prompt. 
    {\bf Right:} Mean LPB. 
    Shaded regions represent the standard deviation over 20 runs.
    }
    \label{fig:budget_effect}
\end{figure*}
We generate a dataset that simulates a high-risk setting with predominantly prompts that can yield unsafe outcomes. To this end, we generate $n = 100{,}000$ pairs \((X_i, p_i)\), where \(p_i\) is the true probability of sampling an unsafe \(Y_i\) for \(X_i\in\mathbb{R}^d\), with $d=10$.  Specifically,
\(
X_i \mid p_i,
\)
follows a normal distribution, defined in Appendix~\ref{subsec:synthetic_experiments}. The unsafe probability $p$ is chosen to simulate a dataset of ``suspicious'' prompts, where 90\% of which have a high probability for unsafe generation, with the remaining 10\% of prompts having a low probability for such an event.
Then, we draw $T_i\sim\mathrm{Geom}(p_i)$. We randomly partition the data into training (45\%), calibration (45\%), and test (10\%) sets. For each example in the training set, we generate $500$ outputs. We then use this training set to fit the quantile estimators $\hat{q}_\tau(x)$.

We compare (1) the \texttt{Uncalibrated} baseline, the raw quantile estimate $\hat q_{\tau}(X)$ at quantile level $\tau=\alpha$; (2) the \texttt{Naive} calibration method, which is a slight variant of the method from Section~\ref{sec:naive_allocation}. The original \texttt{Naive} approach tended to be unstable, often yielding trivial LPBs, since $\hat{q}_\tau(X_i)\leq C_i$ holds for very few points at high quantile levels and these samples do not contribute to the miscoverage estimator in~\eqref{eq:miscoverage_estimator}.
To mitigate this, we restrict 
the $\tau$'s search space to $\mathcal{T}\cap[0,\tauprior]$, as in the adaptive budget allocation methods. We also include (3) the \texttt{Basic} method from Section~\ref{subsec:prompt_adaptive}; (4) its \texttt{Trimmed} version from Section~\ref{subsec:prompt_capped}; and (5) our flagship \texttt{Optimized} approach from Section~\ref{subsec:prompt_global}. 

Recall that the calibration procedures are inherently random due to the sampling of $C_i$ and the random generations obtained by the model $\mathcal{G}(x)$. Ideally, this randomness should have a small effect on our bounds. To quantify this randomness, we fix the data split, and run each method for $J=20$ random draws of censoring and time-to-unsafe-sampling times. We index each run by $j$, and denote the resulting LPB by $\hat q_{\hat\tau}^{(j)}(X_i)$ and compute: 
\begin{align*}
    \mathrm{AvgCoverage}^{(j)}
&= \frac{1}{|\mathcal I_{\mathrm{test}}|}
  \sum_{i\in \mathcal I_{\mathrm{test}}}
    \mathbb{P}\{T_i \le \hat q_{\hat\tau}^{(j)}(X_i) | X_i\}, \\\
     \mathrm{AvgLPB}^{(j)}
&= \frac{1}{|\mathcal I_{\mathrm{test}}|}
  \sum_{i\in \mathcal I_{\mathrm{test}}}
    \hat q_{\hat\tau}^{(j)}(X_i), \\
    \mathrm{AvgBudget}^{(j)}
&= \frac{1}{| \mathcal I_{\mathrm{test}}|}
  \sum_{i\in \mathcal I_{\mathrm{test}}}
    C_i^{(j)}.
\end{align*}
Above, $\mathcal I_{\mathrm{test}}$ are the indices of the test points. We report the average and standard deviation of the above quantities over $J$ runs.

\textbf{Results.} We study the effect of the average budget per prompt $B/|\Itwo|$ on the coverage and mean LPB obtained by each method. Following Figure~\ref{fig:budget_effect}, we can see that the \texttt{Uncalibrated} method does not attain valid coverage. The \ttnaive method severely undercovers $T$ for small budgets, and overcovers $T$ for larger budgets. The \texttt{Basic} method attains more stable coverage around the desired $90\%$ level under medium-to-high budget, but exhibits high variance. By contrast, the \texttt{Trimmed} method attains nearly $90\%$ coverage across all budget constraints. Our flagship \texttt{Optimized} calibration method attains tight 90\% coverage with minimal variance compared to all other methods. The middle panel highlights that both the \ttnaive and \texttt{Optimized} approaches fully utilize the  budget; however, the latter is tightly regulated around the desired coverage---better utilizing the sampling budget. The right panel demonstrates that the \texttt{Trimmed} and \texttt{Optimized} methods produce stable LPBs with better statistical efficiency as the average budget per prompt increases. The increased budget allows us to set higher quantile trim values $M$ while enforcing low $\gamma$. 
As the budget increases, the two adaptive methods converge to the performance of the \texttt{Basic} technique. However, this technique yields higher LPBs with a higher coverage variability. Additional experiments, including oracle models, varying maximal weights, $\alpha$ levels, and calibration/test splits, are in Appendix~\ref{subsec:synthetic_experiments}.
\subsection{Real Data Experiments}
\label{subsec:real}


\begin{figure*}[ht]
  \centering
  \includegraphics[width=0.8\textwidth]{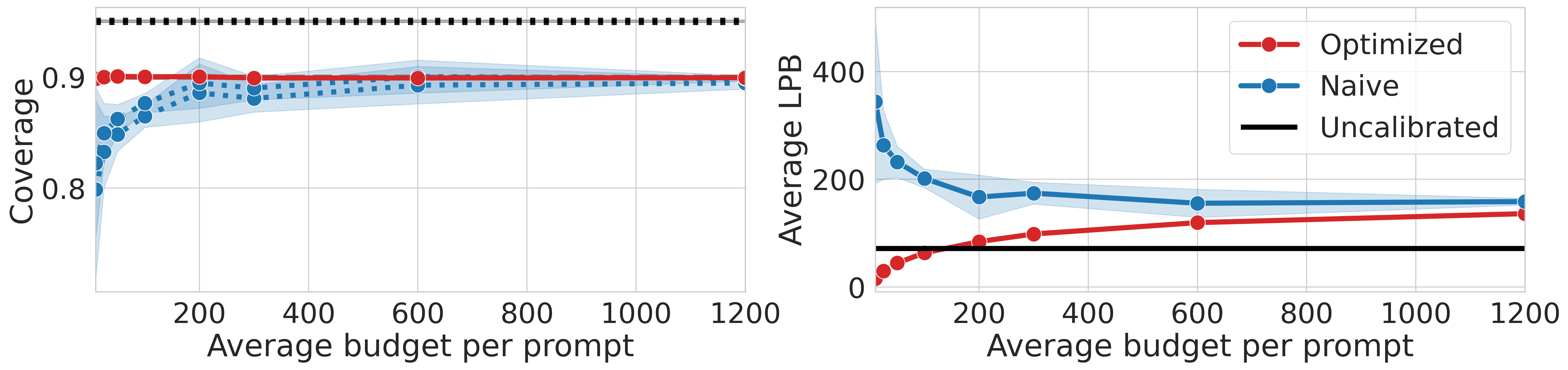}
  \caption{RealToxicityPrompts dataset experiment. {\bf Left:} Empirical coverage rate (target 90\%). The true coverage of the \texttt{Optimized} method is in a solid red line, while the upper and lower bounds on the coverage of the \texttt{Uncalibrated} and \texttt{Naive} methods correspond to the dotted lines. {\bf Right:} Mean LPB. Shaded regions represent the standard deviation, computed over $5$ runs. Higher is better.}
  \label{fig:real_results}
\end{figure*}

We now evaluate the effectiveness of our \texttt{Optimized} calibration method on real-world data. For this purpose, we use the RealToxicityPrompts dataset~\citep{gehman2020realtoxicityprompts}, which contains 99,442 sentences designed to measure model toxicity. Each sentence is divided into two halves: the first half serves as the input prompt $X_i$ to a language model, which must then complete the sentence. In this experiment, we generate about 30 million model responses. Because this scale is computationally demanding, we employ a relatively small model, Llama 3.2 1B~\citep{meta2024llama3_2, grattafiori2024llama3}. Furthermore, as we do not have access to human annotations, we rely on the Detoxify-original classifier~\citep{hanu2020detoxify} to label each completion as safe (non-toxic) or unsafe (toxic), using a toxicity score threshold of 0.5.

We randomly split the dataset into training (50\%), validation (10\%), calibration (20\%), and test (20\%) sets. On the training set, we sample \(N=500\) responses per prompt and fit a quantile estimator $\hat{q}_\tau(x)$ as detailed in Appendix~\ref{sec:details}. We employ the three calibration methods described in Section~\ref{subsec:synth}: (1) the \texttt{Uncalibrated} baseline, (2) the \texttt{Naive} method, and (3) the \texttt{Optimized} method, all aiming to attain 90\% coverage rate. 

In all experiments, we report the empirical mean of the LPBs over the test set. We also estimate the empirical miscoverage by drawing
$\min\bigl(\hat L(X_i),2400\bigr)$ samples for each prompt. We cap the maximal LLM generations on the test set at 2,400 samplings per prompt due to computational constraints. Since the \texttt{Optimized} method trims the quantiles at $M<2400$ across all evaluated budget levels, by drawing exactly $\hat L(X_i)$ samples (with $\hat L(X_i)\le M$) we can compute the miscoverage rate directly. Since $\hat L(X_i)$ is unbounded for the \texttt{Uncalibrated} and \texttt{Naive} methods, we can only report a lower bound on the empirical coverage rate using the 2,400 generations; see also \citep{gui2024conformalized}. In turn, the empirical lower and upper bounds are given by
     $
       \frac{1}{|\mathcal I_{\mathrm{test}}|}\sum_{i\in\mathcal I_{\mathrm{test}}}
       \mathbb I\bigl\{\hat L(X_i)\le\tilde T_i\bigr\},
     $ and $
       \frac{1}{|\mathcal I_{\mathrm{test}}|}\sum_{i\in\mathcal I_{\mathrm{test}}}
       \mathbb I\bigl\{\min(\hat L(X_i),2400)\le\tilde T_i\bigr\}
     $, respectively.
This evaluation protocol offers a fair comparison under relatively limited computations. To measure the variability introduced by the inherently stochastic calibration procedures, we use the same fixed data split for each run, as in our synthetic experiments. 

\textbf{Results.} Figure~\ref{fig:real_results} presents the results across \(J=5\) runs. The \texttt{Uncalibrated} baseline overcovers the time-to-unsafe-sampling, resulting in overly conservative LPBs. The \texttt{Naive} method, on the other hand, generates invalid LPBs with high variance at small budget levels. In contrast, our \texttt{Optimized} method consistently attains near‐nominal coverage rates with low variance. Notably, the LPBs produced by the \texttt{Optimized} method increase as the budget increases, approaching the values of the \texttt{Naive} and surpassing those of the \texttt{Uncalibrated} method.
Overall, this experiment shows that the \texttt{Optimized} method produces LPBs that are both valid across all tested budgets, and become more informative as the budget grows.

\section{Discussion}\label{sec:discussion}

We introduced a novel approach to construct a calibrated LPB for time-to-unsafe-sampling of a given prompt, casting this task as a survival analysis problem. Key to our method is the design of an optimized censoring mechanism that utilizes a given budget constraint on the number of LLM generations.
One limitation of our approach is that the coverage guarantee is marginal and does not hold conditionally on a specific subgroup. Consequently, one cannot group prompts with especially low (or high) LPBs and assume that those bounds retain their nominal coverage when conditioned on this selection. 
In future work, we plan to extend our approach to this conditional inference setting, leveraging the methods by~\citet{sesia2025conformal, jin2025confidence} as a foundation for achieving valid group-conditional coverage.
We remark that in our experimental framework, we treat the outputs of the Detoxify audit function as ground-truth events, as obtaining human annotations at the scale of millions of generations is infeasible. This reliance on a surrogate model introduces the potential for label noise. Consequently, a valuable direction for future research is to investigate how training on proxy labels affects the validity of the LPB when evaluated against human-annotated test data, potentially leveraging recent theoretical advancements in conformal prediction with label noise~\cite{einbinder2024label}.
Another limitation is that our calibration assumes the samples are drawn i.i.d. Even in benign real-world setups, the prompt distribution can shift, potentially undermining the validity of our method. A promising direction, which we have already begun to explore, we aim to propose continual or adaptive recalibration mechanisms that adjust to evolving user behavior and emergent risks, possibly by drawing on ideas from~\citep{gibbs2024conformal}. Concurrently, we are investigating how to optimally allocate the budget of the training phase to maximize the model's accuracy under limited LLM-sampling resources.

\subsubsection*{Acknowledgments}
Y.~R., S.~F., and H.~D. were supported by the European Union (ERC, SafetyBounds, 101163414). Views and opinions expressed are however those of the authors only and do not necessarily reflect those of the European Union or the European Research Council Executive Agency. Neither the European Union nor the granting authority can be held responsible for them.
Y.R. thanks the Career Advancement Fellowship, Technion.

\bibliographystyle{apalike}
\bibliography{references}

\begin{thebibliography}{}

\bibitem[Ahmed et~al., 2025]{ahmed2024controllable}
Ahmed, K., Chang, K.-W., and Broeck, G. V.~d. (2025).
\newblock {Controllable Generation via Locally Constrained Resampling}.
\newblock {\em International Conference on Learning Representations}.

\bibitem[Bhatt et~al., 2025]{bhatt2025ctrl}
Bhatt, A., Rushing, C., Kaufman, A., Tracy, T., Georgiev, V., Matolcsi, D., Khan, A., and Shlegeris, B. (2025).
\newblock {Ctrl-Z: Controlling AI Agents via Resampling}.
\newblock {\em arXiv preprint arXiv:2504.10374}.

\bibitem[Bianchi et~al., 2024]{bianchi2024safety}
Bianchi, F., Suzgun, M., Attanasio, G., Rottger, P., Jurafsky, D., Hashimoto, T., and Zou, J. (2024).
\newblock {Safety-Tuned LLaMAs: Lessons From Improving the Safety of Large Language Models that Follow Instructions}.
\newblock In {\em International Conference on Learning Representations}.

\bibitem[Cand{\`e}s et~al., 2023]{candes2023conformalized}
Cand{\`e}s, E., Lei, L., and Ren, Z. (2023).
\newblock Conformalized survival analysis.
\newblock {\em Journal of the Royal Statistical Society Series B: Statistical Methodology}, 85(1):24--45.

\bibitem[Chao et~al., 2025]{chao2025jailbreaking}
Chao, P., Robey, A., Dobriban, E., Hassani, H., Pappas, G.~J., and Wong, E. (2025).
\newblock Jailbreaking black box large language models in twenty queries.
\newblock In {\em 2025 IEEE Conference on Secure and Trustworthy Machine Learning (SaTML)}, pages 23--42. IEEE.

\bibitem[Cherian et~al., 2024]{cherian2024large}
Cherian, J., Gibbs, I., and Cand{\`e}s, E. (2024).
\newblock Large language model validity via enhanced conformal prediction methods.
\newblock In {\em Advances in Neural Information Processing Systems}.

\bibitem[Davidov et~al., 2025]{davidovconformalized}
Davidov, H., Feldman, S., Shamai, G., Kimmel, R., and Romano, Y. (2025).
\newblock {Conformalized Survival Analysis for General Right-Censored Data}.
\newblock In {\em International Conference on Learning Representations}.

\bibitem[Dickson et~al., 2024]{dickson2024estimating}
Dickson, L., Campbell, A., et~al. (2024).
\newblock Estimating the probabilities of rare outputs in language models.
\newblock {\em arXiv preprint arXiv:2410.13211}.

\bibitem[Einbinder et~al., 2024]{einbinder2024label}
Einbinder, B.-S., Feldman, S., Bates, S., Angelopoulos, A.~N., Gendler, A., and Romano, Y. (2024).
\newblock Label noise robustness of conformal prediction.
\newblock {\em Journal of Machine Learning Research}, 25(328):1--66.

\bibitem[Gehman et~al., 2020]{gehman2020realtoxicityprompts}
Gehman, S., Gururangan, S., Sap, M., Choi, Y., and Smith, N.~A. (2020).
\newblock {RealToxicityPrompts: Evaluating Neural Toxic Degeneration in Language Models}.
\newblock In {\em Findings of the Association for Computational Linguistics: EMNLP}, pages 3356--3369.

\bibitem[Gibbs and Cand{\`e}s, 2024]{gibbs2024conformal}
Gibbs, I. and Cand{\`e}s, E.~J. (2024).
\newblock {Conformal Inference for Online Prediction with Arbitrary Distribution Shifts}.
\newblock {\em Journal of Machine Learning Research}, 25(162):1--36.

\bibitem[Grattafiori et~al., 2024]{grattafiori2024llama3}
Grattafiori, A., Dubey, A., Jauhri, A., Pandey, A., Kadian, A., Al-Dahle, A., Letman, A., Mathur, A., Schelten, A., Vaughan, A., Yang, A., Fan, A., Goyal, A., Hartshorn, A., Yang, A., Mitra, A., Sravankumar, A., Korenev, A., Hinsvark, A., Rao, A., Zhang, A., Rodriguez, A., Gregerson, A., Spataru, A., Roziere, B., Biron, B., Tang, B., Chern, B., Caucheteux, C., Nayak, C., Bi, C., Marra, C., McConnell, C., Keller, C., Touret, C., Wu, C., Wong, C., Ferrer, C.~C., Nikolaidis, C., Allonsius, D., Song, D., Pintz, D., Livshits, D., Wyatt, D., Esiobu, D., Choudhary, D., Mahajan, D., Garcia-Olano, D., Perino, D., Hupkes, D., Lakomkin, E., AlBadawy, E., Lobanova, E., Dinan, E., Smith, E.~M., Radenovic, F., Guzmán, F., Zhang, F., Synnaeve, G., Lee, G., Anderson, G.~L., Thattai, G., Nail, G., Mialon, G., Pang, G., Cucurell, G., Nguyen, H., Korevaar, H., Xu, H., Touvron, H., Zarov, I., Ibarra, I.~A., Kloumann, I., Misra, I., Evtimov, I., Zhang, J., Copet, J., Lee, J., Geffert, J., Vranes, J., Park, J., Mahadeokar, J.,
  Shah, J., van~der Linde, J., Billock, J., Hong, J., Lee, J., Fu, J., Chi, J., Huang, J., Liu, J., Wang, J., Yu, J., Bitton, J., Spisak, J., Park, J., Rocca, J., Johnstun, J., Saxe, J., Jia, J., Alwala, K.~V., Prasad, K., Upasani, K., Plawiak, K., Li, K., Heafield, K., Stone, K., El-Arini, K., Iyer, K., Malik, K., Chiu, K., Bhalla, K., Lakhotia, K., Rantala-Yeary, L., van~der Maaten, L., Chen, L., Tan, L., Jenkins, L., Martin, L., Madaan, L., Malo, L., Blecher, L., Landzaat, L., de~Oliveira, L., Muzzi, M., Pasupuleti, M., Singh, M., Paluri, M., Kardas, M., Tsimpoukelli, M., Oldham, M., Rita, M., Pavlova, M., Kambadur, M., Lewis, M., Si, M., Singh, M.~K., Hassan, M., Goyal, N., Torabi, N., Bashlykov, N., Bogoychev, N., Chatterji, N., Zhang, N., Duchenne, O., Çelebi, O., Alrassy, P., Zhang, P., Li, P., Vasic, P., Weng, P., Bhargava, P., Dubal, P., Krishnan, P., Koura, P.~S., Xu, P., He, Q., Dong, Q., Srinivasan, R., Ganapathy, R., Calderer, R., Cabral, R.~S., Stojnic, R., Raileanu, R., Maheswari, R., Girdhar,
  R., Patel, R., Sauvestre, R., Polidoro, R., Sumbaly, R., Taylor, R., Silva, R., Hou, R., Wang, R., Hosseini, S., Chennabasappa, S., Singh, S., Bell, S., Kim, S.~S., Edunov, S., Nie, S., Narang, S., Raparthy, S., Shen, S., Wan, S., Bhosale, S., Zhang, S., Vandenhende, S., Batra, S., Whitman, S., Sootla, S., Collot, S., Gururangan, S., Borodinsky, S., Herman, T., Fowler, T., Sheasha, T., Georgiou, T., Scialom, T., Speckbacher, T., Mihaylov, T., Xiao, T., Karn, U., Goswami, V., Gupta, V., Ramanathan, V., Kerkez, V., Gonguet, V., Do, V., Vogeti, V., Albiero, V., Petrovic, V., Chu, W., Xiong, W., Fu, W., Meers, W., Martinet, X., Wang, X., Wang, X., Tan, X.~E., Xia, X., Xie, X., Jia, X., Wang, X., Goldschlag, Y., Gaur, Y., Babaei, Y., Wen, Y., Song, Y., Zhang, Y., Li, Y., Mao, Y., Coudert, Z.~D., Yan, Z., Chen, Z., Papakipos, Z., Singh, A., Srivastava, A., Jain, A., Kelsey, A., Shajnfeld, A., Gangidi, A., Victoria, A., Goldstand, A., Menon, A., Sharma, A., Boesenberg, A., Baevski, A., Feinstein, A., Kallet, A.,
  Sangani, A., Teo, A., Yunus, A., Lupu, A., Alvarado, A., Caples, A., Gu, A., Ho, A., Poulton, A., Ryan, A., Ramchandani, A., Dong, A., Franco, A., Goyal, A., Saraf, A., Chowdhury, A., Gabriel, A., Bharambe, A., Eisenman, A., Yazdan, A., James, B., Maurer, B., Leonhardi, B., Huang, B., Loyd, B., Paola, B.~D., Paranjape, B., Liu, B., Wu, B., Ni, B., Hancock, B., Wasti, B., Spence, B., Stojkovic, B., Gamido, B., Montalvo, B., Parker, C., Burton, C., Mejia, C., Liu, C., Wang, C., Kim, C., Zhou, C., Hu, C., Chu, C.-H., Cai, C., Tindal, C., Feichtenhofer, C., Gao, C., Civin, D., Beaty, D., Kreymer, D., Li, D., Adkins, D., Xu, D., Testuggine, D., David, D., Parikh, D., Liskovich, D., Foss, D., Wang, D., Le, D., Holland, D., Dowling, E., Jamil, E., Montgomery, E., Presani, E., Hahn, E., Wood, E., Le, E.-T., Brinkman, E., Arcaute, E., Dunbar, E., Smothers, E., Sun, F., Kreuk, F., Tian, F., Kokkinos, F., Ozgenel, F., Caggioni, F., Kanayet, F., Seide, F., Florez, G.~M., Schwarz, G., Badeer, G., Swee, G., Halpern, G.,
  Herman, G., Sizov, G., Guangyi, Zhang, Lakshminarayanan, G., Inan, H., Shojanazeri, H., Zou, H., Wang, H., Zha, H., Habeeb, H., Rudolph, H., Suk, H., Aspegren, H., Goldman, H., Zhan, H., Damlaj, I., Molybog, I., Tufanov, I., Leontiadis, I., Veliche, I.-E., Gat, I., Weissman, J., Geboski, J., Kohli, J., Lam, J., Asher, J., Gaya, J.-B., Marcus, J., Tang, J., Chan, J., Zhen, J., Reizenstein, J., Teboul, J., Zhong, J., Jin, J., Yang, J., Cummings, J., Carvill, J., Shepard, J., McPhie, J., Torres, J., Ginsburg, J., Wang, J., Wu, K., U, K.~H., Saxena, K., Khandelwal, K., Zand, K., Matosich, K., Veeraraghavan, K., Michelena, K., Li, K., Jagadeesh, K., Huang, K., Chawla, K., Huang, K., Chen, L., Garg, L., A, L., Silva, L., Bell, L., Zhang, L., Guo, L., Yu, L., Moshkovich, L., Wehrstedt, L., Khabsa, M., Avalani, M., Bhatt, M., Mankus, M., Hasson, M., Lennie, M., Reso, M., Groshev, M., Naumov, M., Lathi, M., Keneally, M., Liu, M., Seltzer, M.~L., Valko, M., Restrepo, M., Patel, M., Vyatskov, M., Samvelyan, M., Clark,
  M., Macey, M., Wang, M., Hermoso, M.~J., Metanat, M., Rastegari, M., Bansal, M., Santhanam, N., Parks, N., White, N., Bawa, N., Singhal, N., Egebo, N., Usunier, N., Mehta, N., Laptev, N.~P., Dong, N., Cheng, N., Chernoguz, O., Hart, O., Salpekar, O., Kalinli, O., Kent, P., Parekh, P., Saab, P., Balaji, P., Rittner, P., Bontrager, P., Roux, P., Dollar, P., Zvyagina, P., Ratanchandani, P., Yuvraj, P., Liang, Q., Alao, R., Rodriguez, R., Ayub, R., Murthy, R., Nayani, R., Mitra, R., Parthasarathy, R., Li, R., Hogan, R., Battey, R., Wang, R., Howes, R., Rinott, R., Mehta, S., Siby, S., Bondu, S.~J., Datta, S., Chugh, S., Hunt, S., Dhillon, S., Sidorov, S., Pan, S., Mahajan, S., Verma, S., Yamamoto, S., Ramaswamy, S., Lindsay, S., Lindsay, S., Feng, S., Lin, S., Zha, S.~C., Patil, S., Shankar, S., Zhang, S., Zhang, S., Wang, S., Agarwal, S., Sajuyigbe, S., Chintala, S., Max, S., Chen, S., Kehoe, S., Satterfield, S., Govindaprasad, S., Gupta, S., Deng, S., Cho, S., Virk, S., Subramanian, S., Choudhury, S.,
  Goldman, S., Remez, T., Glaser, T., Best, T., Koehler, T., Robinson, T., Li, T., Zhang, T., Matthews, T., Chou, T., Shaked, T., Vontimitta, V., Ajayi, V., Montanez, V., Mohan, V., Kumar, V.~S., Mangla, V., Ionescu, V., Poenaru, V., Mihailescu, V.~T., Ivanov, V., Li, W., Wang, W., Jiang, W., Bouaziz, W., Constable, W., Tang, X., Wu, X., Wang, X., Wu, X., Gao, X., Kleinman, Y., Chen, Y., Hu, Y., Jia, Y., Qi, Y., Li, Y., Zhang, Y., Zhang, Y., Adi, Y., Nam, Y., Yu, Wang, Zhao, Y., Hao, Y., Qian, Y., Li, Y., He, Y., Rait, Z., DeVito, Z., Rosnbrick, Z., Wen, Z., Yang, Z., Zhao, Z., and Ma, Z. (2024).
\newblock {The Llama 3 Herd of Models}.
\newblock {\em arXiv preprint arXiv:2407.21783}.

\bibitem[Gui et~al., 2024a]{gui2024conformalized}
Gui, Y., Hore, R., Ren, Z., and Barber, R.~F. (2024a).
\newblock Conformalized survival analysis with adaptive cut-offs.
\newblock {\em Biometrika}, 111(2):459--477.

\bibitem[Gui et~al., 2024b]{gui2024conformalllm}
Gui, Y., Jin, Y., and Ren, Z. (2024b).
\newblock {Conformal Alignment: Knowing When to Trust Foundation Models with Guarantees}.
\newblock {\em Advances in Neural Information Processing Systems}.

\bibitem[Hahn et~al., 2011]{hahn2011adaptive}
Hahn, J., Hirano, K., and Karlan, D. (2011).
\newblock Adaptive experimental design using the propensity score.
\newblock {\em Journal of Business \& Economic Statistics}, 29(1):96--108.

\bibitem[Hanu and team, 2020]{hanu2020detoxify}
Hanu, L. and team, U. (2020).
\newblock Detoxify.
\newblock \url{https://github.com/unitaryai/detoxify}.

\bibitem[Hu et~al., 2024]{hutoxicity}
Hu, Z., Piet, J., Zhao, G., Jiao, J., and Wagner, D. (2024).
\newblock {Toxicity Detection for Free}.
\newblock In {\em Advances in Neural Information Processing Systems}.

\bibitem[Inan et~al., 2023]{DBLP:journals/corr/abs-2312-06674}
Inan, H., Upasani, K., Chi, J., Rungta, R., Iyer, K., Mao, Y., Tontchev, M., Hu, Q., Fuller, B., Testuggine, D., and Khabsa, M. (2023).
\newblock {Llama Guard: LLM-based Input-Output Safeguard for Human-AI Conversations}.
\newblock {\em arXiv preprint arXiv:2312.06674}.

\bibitem[Ishwaran et~al., 2008]{ishwaran2008random}
Ishwaran, H., Kogalur, U.~B., Blackstone, E.~H., and Lauer, M.~S. (2008).
\newblock Random survival forests.
\newblock {\em The Annals of Applied Statistics}, pages 841--860.

\bibitem[Jin and Ren, 2025]{jin2025confidence}
Jin, Y. and Ren, Z. (2025).
\newblock Confidence on the focal: Conformal prediction with selection-conditional coverage.
\newblock {\em Journal of the Royal Statistical Society Series B: Statistical Methodology}.

\bibitem[Kaddour et~al., 2023]{kaddour2023challenges}
Kaddour, J., Harris, J., Mozes, M., Bradley, H., Raileanu, R., and McHardy, R. (2023).
\newblock Challenges and applications of large language models.
\newblock {\em arXiv preprint arXiv:2307.10169}.

\bibitem[Kalbfleisch and Prentice, 2011]{kalbfleisch2011statistical}
Kalbfleisch, J.~D. and Prentice, R.~L. (2011).
\newblock {\em The statistical analysis of failure time data}.
\newblock John Wiley \& Sons.

\bibitem[Kaplan and Meier, 1958]{kaplan1958nonparametric}
Kaplan, E.~L. and Meier, P. (1958).
\newblock Nonparametric estimation from incomplete observations.
\newblock {\em Journal of the American statistical association}, 53(282):457--481.

\bibitem[Katzman et~al., 2018]{katzman2018deepsurv}
Katzman, J.~L., Shaham, U., Cloninger, A., Bates, J., Jiang, T., and Kluger, Y. (2018).
\newblock {DeepSurv}: personalized treatment recommender system using a cox proportional hazards deep neural network.
\newblock {\em BMC Medical Research Methodology}, 18(1):24.

\bibitem[Kingma and Ba, 2015]{adam}
Kingma, D.~P. and Ba, J. (2015).
\newblock {Adam: A Method for Stochastic Optimization}.
\newblock In {\em International Conference on Learning Representations}.

\bibitem[Kolmogorov, 1968]{1054210}
Kolmogorov, A. (1968).
\newblock Logical basis for information theory and probability theory.
\newblock {\em IEEE Transactions on Information Theory}, 14(5):662--664.

\bibitem[Kolmogorov, 1983]{ANKolmogorov_1983}
Kolmogorov, A. (1983).
\newblock Combinatorial foundations of information theory and the calculus of probabilities.
\newblock {\em Russian Mathematical Surveys}, 38(4):29.

\bibitem[Kwon et~al., 2023]{kwon2023efficient}
Kwon, W., Li, Z., Zhuang, S., Sheng, Y., Zheng, L., Yu, C.~H., Gonzalez, J.~E., Zhang, H., and Stoica, I. (2023).
\newblock Efficient memory management for large language model serving with pagedattention.
\newblock In {\em Proceedings of the ACM SIGOPS 29th Symposium on Operating Systems Principles}.

\bibitem[Lin et~al., 1998]{lin1998accelerated}
Lin, D., Wei, L., and Ying, Z. (1998).
\newblock Accelerated failure time models for counting processes.
\newblock {\em Biometrika}, 85(3):605--618.

\bibitem[Lin et~al., 2024]{lin2024generating}
Lin, Z., Trivedi, S., and Sun, J. (2024).
\newblock Generating with confidence: Uncertainty quantification for black-box large language models.
\newblock In {\em International Conference on Learning Representations}.

\bibitem[List et~al., 2011]{list2011so}
List, J.~A., Sadoff, S., and Wagner, M. (2011).
\newblock So you want to run an experiment, now what? some simple rules of thumb for optimal experimental design.
\newblock {\em Experimental Economics}, 14:439--457.

\bibitem[Machin et~al., 2006]{machin2006survival}
Machin, D., Cheung, Y.~B., and Parmar, M. (2006).
\newblock {\em Survival analysis: a practical approach}.
\newblock John Wiley \& Sons.

\bibitem[Madaan et~al., 2023]{madaan2023self}
Madaan, A., Tandon, N., Gupta, P., Hallinan, S., Gao, L., Wiegreffe, S., Alon, U., Dziri, N., Prabhumoye, S., Yang, Y., et~al. (2023).
\newblock {Self-refine: Iterative refinement with self-feedback}.
\newblock {\em Advances in Neural Information Processing Systems}.

\bibitem[{Meta AI}, 2024]{meta2024llama3_2}
{Meta AI} (2024).
\newblock {LLaMA 3.2 (1B, 3B)}.
\newblock https://huggingface.co/meta-llama/Llama-3.2-1B.

\bibitem[Mohri and Hashimoto, 2024]{pmlr-v235-mohri24a}
Mohri, C. and Hashimoto, T. (2024).
\newblock {Language Models with Conformal Factuality Guarantees}.
\newblock In {\em International Conference on Machine Learning}.

\bibitem[Nagpal et~al., 2021]{nagpal2021deep}
Nagpal, C., Yadlowsky, S., Rostamzadeh, N., and Heller, K. (2021).
\newblock {Deep Cox Mixtures for Survival Regression}.
\newblock In {\em Machine Learning for Healthcare Conference}, pages 674--708. PMLR.

\bibitem[Papadopoulos, 2008]{papadopoulos2008inductive}
Papadopoulos, H. (2008).
\newblock Inductive conformal prediction: Theory and application to neural networks.
\newblock In {\em Tools in artificial intelligence}. Citeseer.

\bibitem[Papadopoulos et~al., 2002]{Papadopoulos2002InductiveCM}
Papadopoulos, H., Proedrou, K., Vovk, V., and Gammerman, A. (2002).
\newblock Inductive confidence machines for regression.
\newblock In {\em European Conference on Machine Learning}.

\bibitem[Perez et~al., 2022]{perez2022red}
Perez, E., Huang, S., Song, F., Cai, T., Ring, R., Aslanides, J., Glaese, A., McAleese, N., and Irving, G. (2022).
\newblock Red teaming language models with language models.
\newblock {\em arXiv preprint arXiv:2202.03286}.

\bibitem[Prinster et~al., 2024]{prinster2024conformal}
Prinster, D., Stanton, S., Liu, A., and Saria, S. (2024).
\newblock {Conformal Validity Guarantees Exist for Any Data Distribution (and How to Find Them)}.
\newblock In {\em International Conference on Machine Learning}.

\bibitem[Quach et~al., 2024]{quach2024conformal}
Quach, V., Fisch, A., Schuster, T., Yala, A., Sohn, J.~H., Jaakkola, T.~S., and Barzilay, R. (2024).
\newblock {Conformal Language Modeling}.
\newblock In {\em International Conference on Learning Representations}.

\bibitem[Rao et~al., 2024]{rao2023tricking}
Rao, A., Vashistha, S., Naik, A., Aditya, S., and Choudhury, M. (2024).
\newblock Tricking llms into disobedience: Formalizing, analyzing, and detecting jailbreaks.
\newblock In {\em Proceedings of the 2024 Joint International Conference on Computational Linguistics, Language Resources and Evaluation (LREC-COLING 2024)}, pages 16802--16830.
\newblock First released on arXiv in 2023.

\bibitem[Ray, 2023]{RAY2023121}
Ray, P.~P. (2023).
\newblock {ChatGPT: A comprehensive review on background, applications, key challenges, bias, ethics, limitations and future scope}.
\newblock {\em Internet of Things and Cyber-Physical Systems}, 3:121--154.

\bibitem[Sesia and Svetnik, 2025]{sesia2025conformal}
Sesia, M. and Svetnik, V. (2025).
\newblock Conformal survival bands for risk screening under right-censoring.
\newblock {\em arXiv preprint arXiv:2505.04568}.

\bibitem[Settles, 2010]{settles2009active}
Settles, B. (2010).
\newblock Active learning literature survey.
\newblock {\em University of Wisconsin, Madison}, 52.

\bibitem[Shinn et~al., 2023]{shinn2023reflexion}
Shinn, N., Cassano, F., Gopinath, A., Narasimhan, K., and Yao, S. (2023).
\newblock {Reflexion: Language agents with verbal reinforcement learning}.
\newblock {\em Advances in Neural Information Processing Systems}.

\bibitem[Stroebl et~al., 2024]{stroebl2024inference}
Stroebl, B., Kapoor, S., and Narayanan, A. (2024).
\newblock {Inference Scaling fLaws: The Limits of LLM Resampling with Imperfect Verifiers}.
\newblock {\em arXiv preprint arXiv:2411.17501}.

\bibitem[Su et~al., 2024]{su2024mission}
Su, J., Kempe, J., and Ullrich, K. (2024).
\newblock Mission impossible: A statistical perspective on jailbreaking llms.
\newblock {\em Advances in Neural Information Processing Systems}, 37:38267--38306.

\bibitem[Tibshirani et~al., 2019]{tibshirani2020conformalpredictioncovariateshift}
Tibshirani, R.~J., Foygel~Barber, R., Cand{\`e}s, E., and Ramdas, A. (2019).
\newblock {Conformal Prediction Under Covariate Shift}.
\newblock In {\em Advances in Neural Information Processing Systems}.

\bibitem[Vidgen et~al., 2023]{Vidgen2023SimpleSafetyTestsAT}
Vidgen, B., Scherrer, N., Kirk, H.~R., Qian, R., Kannappan, A., Hale, S.~A., and R{\"o}ttger, P. (2023).
\newblock {SimpleSafetyTests: a Test Suite for Identifying Critical Safety Risks in Large Language Models}.
\newblock {\em arXiv preprint arXiv:2311.08370}.

\bibitem[Vovk et~al., 2005]{vovk2005algorithmic}
Vovk, V., Gammerman, A., and Shafer, G. (2005).
\newblock {\em {Algorithmic Learning in a Random World}}, volume~29.
\newblock Springer.

\bibitem[Wang et~al., 2023]{wangvoyager}
Wang, G., Xie, Y., Jiang, Y., Mandlekar, A., Xiao, C., Zhu, Y., Fan, L., and Anandkumar, A. (2023).
\newblock Voyager: An open-ended embodied agent with large language models.
\newblock {\em Transactions on Machine Learning Research}.

\bibitem[Wang et~al., 2025]{wang2025copu}
Wang, S., Jiang, Y., Tang, Y., Cheng, L., and Chen, H. (2025).
\newblock Copu: Conformal prediction for uncertainty quantification in natural language generation.
\newblock {\em arXiv preprint arXiv:2502.12601}.

\bibitem[Wang et~al., 2024]{wang2024conu}
Wang, Z., Duan, J., Cheng, L., Zhang, Y., Wang, Q., Shi, X., Xu, K., Shen, H., and Zhu, X. (2024).
\newblock Conu: Conformal uncertainty in large language models with correctness coverage guarantees.
\newblock {\em arXiv preprint arXiv:2407.00499}.

\bibitem[Warner et~al., 2024]{warner2024smarter}
Warner, B., Chaffin, A., Clavié, B., Weller, O., Hallström, O., Taghadouini, S., Gallagher, A., Biswas, R., Ladhak, F., Aarsen, T., Cooper, N., Adams, G., Howard, J., and Poli, I. (2024).
\newblock {Smarter, Better, Faster, Longer: A Modern Bidirectional Encoder for Fast, Memory Efficient, and Long Context Finetuning and Inference}.
\newblock {\em arXiv preprint arXiv:2412.13663}.

\bibitem[Westerlund, 2019]{westerlund2019emergence}
Westerlund, M. (2019).
\newblock The emergence of deepfake technology: A review.
\newblock {\em Technology innovation management review}, 9(11).

\bibitem[Xu et~al., 2025]{xu2025utilizing}
Xu, W., Wei, Z., Sun, X., Zhang, D., Yang, D., Zou, Q., and Zhang, X. (2025).
\newblock Utilizing jailbreak probability to attack and safeguard multimodal llms.
\newblock {\em arXiv preprint arXiv:2503.06989}.

\bibitem[Zrnic and Cand{\`e}s, 2024]{zrnic2024active}
Zrnic, T. and Cand{\`e}s, E. (2024).
\newblock {Active Statistical Inference}.
\newblock In {\em International Conference on Machine Learning}.

\end{thebibliography}

\section*{Checklist}

\begin{enumerate}

  \item For all models and algorithms presented, check if you include:
  \begin{enumerate}
    \item A clear description of the mathematical setting, assumptions, algorithm, and/or model. [Yes, all of our theory is given in Appendix~\ref{sec:proofs}. Our algorithms are provided in Appendix~\ref{sec:calibration_algs}.]
    \item An analysis of the properties and complexity (time, space, sample size) of any algorithm. [Yes, the time complexity and data sizes are given in Appendix~\ref{subsec:real_data_experiments}]
    \item (Optional) Anonymized source code, with specification of all dependencies, including external libraries. [Yes, our code is anonymized.]
  \end{enumerate}

  \item For any theoretical claim, check if you include:
  \begin{enumerate}
    \item Statements of the full set of assumptions of all theoretical results. [Yes, all our theoretical statements contain the full set of assumptions, see Appendix~\ref{sec:proofs}.]
    \item Complete proofs of all theoretical results. [Yes, all proofs are given in Appendix~\ref{sec:proofs}.]
    \item Clear explanations of any assumptions. [Yes, we explain our assumptions throughout the paper.]     
  \end{enumerate}

  \item For all figures and tables that present empirical results, check if you include:
  \begin{enumerate}
    \item The code, data, and instructions needed to reproduce the main experimental results (either in the supplemental material or as a URL). [Yes, the full experimental details are provided in Appendix~\ref{sec:details}.]
    \item All the training details (e.g., data splits, hyperparameters, how they were chosen). [Yes, all training details are provided in Appendix~\ref{sec:details}.]
    \item A clear definition of the specific measure or statistics and error bars (e.g., with respect to the random seed after running experiments multiple times). [Yes, measures and statistics are provided in Section~\ref{sec:exp} and in Appendix~\ref{sec:details}.]
    \item A description of the computing infrastructure used. (e.g., type of GPUs, internal cluster, or cloud provider). [Yes, see Appendix\ref{sec:machine_spec}.]
  \end{enumerate}

  \item If you are using existing assets (e.g., code, data, models) or curating/releasing new assets, check if you include:
  \begin{enumerate}
    \item Citations of the creator If your work uses existing assets. [Yes, we cited existing assets throughout the paper.]
    \item The license information of the assets, if applicable. [Yes, all assets used are publicly available.]
    \item New assets either in the supplemental material or as a URL, if applicable. [Yes, our code is provided in the supplemental material.]
    \item Information about consent from data providers/curators. [Not Applicable, all assets used are publicly available.]
    \item Discussion of sensible content if applicable, e.g., personally identifiable information or offensive content. [Yes, see Section~\ref{sec:discussion}.]
  \end{enumerate}

  \item If you used crowdsourcing or conducted research with human subjects, check if you include:
  \begin{enumerate}
    \item The full text of instructions given to participants and screenshots. [Not Applicable, we did not use crowdsourcing.]
    \item Descriptions of potential participant risks, with links to Institutional Review Board (IRB) approvals if applicable. [Not Applicable, we did not use crowdsourcing.]
    \item The estimated hourly wage paid to participants and the total amount spent on participant compensation. [Not Applicable, we did not use crowdsourcing.]
  \end{enumerate}

\end{enumerate}

\newpage
\onecolumn
\appendix

\section{Additional Related Work}\label{sec:additional_related_work}

The methods introduced in our paper draw on ideas from survival analysis, conformal prediction, and black-box uncertainty quantification. The {\texttt{Optimized}} method (Section~\ref{subsec:prompt_global}) in particular is also related to:  conformal prediction foundations, uncertainty quantification for LLMs, active learning and statistical inference, and optimal experimental design.

\paragraph{Conformal Prediction.} The theoretical foundations of conformal prediction were established by~\citet{ANKolmogorov_1983, 1054210, Papadopoulos2002InductiveCM, vovk2005algorithmic}. This framework has since been extended to accommodate covariate shift by applying weights to the distribution of conformity scores~\citep{tibshirani2020conformalpredictioncovariateshift}. Similar weighting strategies have also been employed to integrate active learning into conformal prediction~\citep{prinster2024conformal}. Building further on these weighted approaches, conformal prediction methods have recently been adapted for survival analysis tasks~\citep{candes2023conformalized, gui2024conformalized, davidovconformalized}.

\paragraph{Uncertainty Quantification for LLMs.}
Several recent works adapt conformal ideas to off-the-shelf LLMs. For instance,~\citet{wang2024conu} adapts conformal prediction to open‐ended generation by integrating a self–consistency–based uncertainty measure, achieving strict correctness coverage across seven popular LLMs and diverse domains. The work of~\citet{wang2025copu} further extends this by explicitly inserting the ground truth into candidate outputs and using logit‐based nonconformity scores to guarantee coverage over a wider range of error rates, though its reliance on logits raises calibration concerns when only API‐level access is available. Additionally, {\em Conformal Language Modeling} has been developed by \citet{quach2024conformal} to calibrate both stopping and rejection rules to guarantee coverage in open-domain QA and summarization.~\citet{lin2024generating} distinguish uncertainty from confidence using semantic dispersion metrics, enabling selective generation of low-confidence outputs.
To refine validity under practical constraints, \citet{cherian2024large} extended conditional conformal methods to incorporate utility-based guarantees and differentiable filtering. \citet{pmlr-v235-mohri24a} introduced ``conformal factuality,'' constructing entailment-based uncertainty sets that deliver high-probability correctness. The approach proposed in~\citep{gui2024conformalllm}  provides a guarantee for aligned responses in QA and radiology applications.

\paragraph{Active Learning and Statistical Inference.}
Active learning methods aim to select the most informative samples for labeling, improving model accuracy under tight budgets \citep{settles2009active}. This was formalized in~\citep{zrnic2024active} for statistical inference tasks, where one prioritizes labeling points with high model uncertainty while maintaining valid confidence intervals. Our \texttt{Optimized} method relates to this idea as it derives sample probabilities that directly minimize the variance of the augmented inverse-propensity-weighting estimator.

\paragraph{Optimal Experimental Design.}
Optimal experimental design allocates observations to minimize estimator variance under a given model. In~\citep{hahn2011adaptive}, the authors develop a two-stage adaptive design that chooses propensity scores to minimize the asymptotic variance bound in treatment-effect estimation, and the approach presented in~\citep{list2011so} presents simple rules of thumb for efficient designs under resource constraints. Analogously, our \texttt{Optimized} method treats the sampling rule $\pi(x)$ as a design variable and solves for the allocation that minimizes the miscoverage estimation variance.

\paragraph{Classical Survival Analysis.}
The survival analysis problem was first formalized in biostatistics and reliability engineering.~\cite{kaplan1958nonparametric} introduce a nonparametric estimator of the survival function for right‐censored data, laying the groundwork for all subsequent methods. The proportional hazards model was later developed, formalizing a semiparametric regression framework that relates covariates to an event’s hazard rate without specifying its baseline form. As an alternative formulation, accelerated failure time (AFT)~\citep{lin1998accelerated} models assume a parametric form for the log‐survival time, allowing direct modeling of time-to-event via common distributions such as Weibull or log-normal. More recently, machine-learning approaches have been developed for survival prediction. Random survival forests leverage ensemble tree methods to nonparametrically estimate cumulative hazard functions~\citep{ishwaran2008random}. DeepSurv employs a neural-network approximation of the Cox partial likelihood to capture complex covariate interactions~\citep{katzman2018deepsurv}, and Deep Cox Mixtures generalizes this by modeling mixtures of proportional hazards to flexibly adapt to heterogeneous subpopulations~\citep{nagpal2021deep}.

\section{Calibration Algorithms}\label{sec:calibration_algs}

\subsection{Calibration With a Known Censoring Mechanism}\label{sec:calibration_general}
We present the general calibration with known censoring mechanism in Algorithm~\ref{alg:calibration_with_known_censoring}.

\begin{algorithm}
\caption{Calibration With a Known Censoring Mechanism}
\label{alg:calibration_with_known_censoring}
\begin{algorithmic}[1]
\REQUIRE Calibration data $\{X_i, C_i, \tilde{T}_i\}_{i\in\Itwo}$, censoring weights $\{w_\tau(\{X_j\}_{j\in \Itwo}, i)\}_{i\in\Itwo, \tau\in\mathcal{T}}$, quantile estimates $\{\hat{f}_\tau(\cdot)\}_{\tau\in\mathcal{T}}$, target miscoverage rate $\tautarget$, prior quantile $\tauprior$.
\\\medskip

\FOR{$\tau \in \mathcal{T} \cap[0,\tauprior]$}
\STATE $
\hat{\alpha}(\tau) \gets \displaystyle\frac{1}{|\Itwo|}\sum_{i\in\mathcal I_2}
w_\tau(\{X_j\}_{j\in \Itwo}, i)\;
\mathbb{I}\bigl\{\tilde{T}_i<\hat f_\tau(X_i) \le C_i\bigr\}
$ {\color{darkgray} // miscoverage est.}

\ENDFOR
\\\medskip
\STATE $\hat{\tau} \gets \sup\Bigl\{\tau \in \mathcal{T} \cap[0,\tauprior] : \sup_{\substack{\tau' \in \mathcal{T} \\ \tau' \le \tau}} \hat{\alpha}(\tau') \le \tautarget\Bigr\}$ {\color{darkgray} \ \ // calibrated quantile level}

\RETURN Lower predictive bound (LPB) for a test point $X_\textup{test}=x$, given by $\hat{L}(x)= \hat{f}_{\hat{\tau}}(x)$
\end{algorithmic}
\end{algorithm}

\subsection{Calibration With a Known Censoring Mechanism: \ttnaive Mode}\label{sec:calibration_naive}
We describe the \ttnaive calibration approach in Algorithm~\ref{alg:calibration_naive}. First, we provide two auxiliary algorithms to sample outcomes from a generative model.

\begin{algorithm}
\caption{Generate and Audit Generative Model Responses}
\label{alg:sample_calibration_responses}
\begin{algorithmic}[1]
\REQUIRE Calibration data $\{X_i\}_{i\in\Itwo}$, Generative model \(\model(\cdot)\), audit function \(\texttt{Audit}(\cdot)\), Censoring times  $\{C_i\}_{i\in\Itwo}$.
\FOR{each \(i \in \mathcal{I}\)}
    \STATE Initialize \(j \gets 0\)
    \REPEAT
        \STATE \(j \gets j + 1\)
        \STATE Generate and evaluate the output's safety: \(Y_i^j \gets \texttt{Audit}(X_i,\model(X_i))\)
    \UNTIL{\(Y_i^j = 1\) \textbf{or} \(C_i = j\)}
    \STATE Set \(\tilde{T}_i \gets j\)
\ENDFOR
\RETURN \(\{\tilde{T}_i\}_{i \in \mathcal{I}_2}\)
\end{algorithmic}
\end{algorithm}

\begin{algorithm}
\caption{Generate Generative Model Outcomes}
\label{alg:sample_calibration_set}
\begin{algorithmic}[1]
\REQUIRE Calibration data $\{X_i\}_{i\in\Itwo}$, generative model \(\model(\cdot)\), max per-prompt budget \(\{\hat{f}_{\tauprior}(X_i)\}_{i \in \mathcal{I}}\), per-prompt evaluation probability \(\{\cprob_i\}_{i \in \mathcal{I}}\), audit function \(\texttt{Audit}(\cdot)\).
\FOR{each \(i \in \mathcal{I}\)}
    \STATE Draw \(V_i \sim \mathrm{Bernoulli}(\cprob_i)\)
    \STATE Set \(C_i \gets V_i \cdot \hat{f}_{\tauprior}(X_i)\)
\ENDFOR
\STATE Obtain $\{\tilde{T}_i\}_{\{i \in \mathcal{I}_2 \}}$ from Algorithm~\ref{alg:sample_calibration_responses} applied with $\{X_i\}_{i\in\Itwo}$, $\model$, $\{C_i\}_{i\in\Itwo}$.
\RETURN \(\{(\tilde{T}_i, C_i)\}_{i \in \mathcal{I}_2}\).
\end{algorithmic}
\end{algorithm}

\begin{algorithm}
\caption{Calibration With a Known Censoring Mechanism: \ttnaive mode}
\label{alg:calibration_naive}
\begin{algorithmic}[1]
\REQUIRE Calibration data $\{X_i\}_{i\in\Itwo}$, generative model \(\model(\cdot)\), audit function \(\texttt{Audit}(\cdot)\), pre-trained quantile regression model $\{\hat{q}_\tau(\cdot)\}_{\tau\in\mathcal{T}}$, target miscoverage rate $\tautarget$, prior quantile $\tauprior$, total budget $B$.
\\\medskip

    \STATE Sample $C_i \sim \mathrm{Geom}(n/B)$ for all ${i\in\Itwo} $ \\ \medskip
    \STATE Obtain $\{\tilde{T}_i\}_{\{i \in \mathcal{I}_2 \}}$ from Algorithm~\ref{alg:sample_calibration_responses} applied with $\{X_i\}_{i\in\Itwo}$, $\model$, $\{C_i\}_{i\in\Itwo}$. \medskip
    \STATE $w_\tau(\{X_j\}_{j\in \Itwo}, i) \gets \frac{1}{\mathbb{P}[\hat{q}_\tau(X_i) \le C_i|X_i]}, \ i\in\Itwo$.  {\color{darkgray} // compute the weights for the geometric $C_i$}.

\STATE $\{(\tilde{T}_i, C_i)\}_{i\in\Itwo} \gets $ Algorithm~\ref{alg:sample_calibration_set} applied to $\{X_i\}_{i\in\Itwo}$ with $\{\cprob_i\}_{i\in\Itwo}$ and $\{\hat{f}_\tauprior(X_i)\}_{i\in\Itwo}$
\STATE $w(\{X_j\}_{j\in\Itwo},i) \gets \frac{1}{\cprob_i}, \ i\in\Itwo$  

\STATE Obtain $\hat{L}(x)$ from Algorithm~\ref{alg:calibration_with_known_censoring}, applied with $\{X_i, C_i, \tilde{T}_i\}_{i\in\Itwo}$, $\{w_\tau(\{X_j\}_{j\in \Itwo}, i)\}_{i\in\Itwo}$, $\{\hat{q}_\tau(\cdot)\}_{\tau\in\mathcal{T}}$, $\tautarget$, and $\tauprior$.

\RETURN $\hat{L}(x)$, the lower predictive bound (LPB) for a test point $X_\textup{test}=x$.
\end{algorithmic}
\end{algorithm}

\subsection{Calibration With a Known Censoring Mechanism: \texttt{Basic} Mode}\label{sec:calibration_basic}

The calibration procedure with \texttt{Basic} budget allocation is outlined in Algorithm~\ref{alg:calibration_basic}.

\begin{algorithm}
\caption{Calibration With a Known Censoring Mechanism: \texttt{Basic} mode}
\label{alg:calibration_basic}
\begin{algorithmic}[1]
\REQUIRE Calibration data $\{X_i\}_{i\in\Itwo}$, generative model \(\model(\cdot)\), audit function \(\texttt{Audit}(\cdot)\), pre-trained quantile regression model $\{\hat{q}_\tau(\cdot)\}_{\tau\in\mathcal{T}}$, target miscoverage rate $\tautarget$, prior quantile $\tauprior$, total budget $B$.
\\\medskip

    \STATE $\{\hat{f}_\tau(X_i)\}_{\tau\in \mathcal{T},i \in \Itwo} \gets \{\hat{q}_\tau(X_i)\}_{\tau\in \mathcal{T},i \in \Itwo}$
    \STATE $\{\cprob_i\}\gets \{\min(\tfrac{B}{|\Itwo|\;\hat f_\tauprior(X_i)},1)\}$ {\color{darkgray} \ \ // per-prompt evaluation probability (Section~\ref{subsec:prompt_adaptive})} \medskip

\STATE $\{(\tilde{T}_i, C_i)\}_{i\in\Itwo} \gets $ Algorithm~\ref{alg:sample_calibration_set} applied to $\{X_i\}_{i\in\Itwo}$ with $\{\cprob_i\}_{i\in\Itwo}$ and $\{\hat{f}_\tauprior(X_i)\}_{i\in\Itwo}$
\STATE $w_\tau(\{X_j\}_{j\in \Itwo}, i) \gets \frac{1}{\cprob_i}, \ i\in\Itwo$  

\STATE Obtain $\hat{L}(x)$ from Algorithm~\ref{alg:calibration_with_known_censoring}, applied with $\{X_i, C_i, \tilde{T}_i\}_{i\in\Itwo}$, $\{w_\tau(\{X_j\}_{j\in \Itwo}, i)\}_{i\in\Itwo}$, $\{\hat{q}_\tau(\cdot)\}_{\tau\in\mathcal{T}}$, $\tautarget$, and $\tauprior$.

\RETURN $\hat{L}(x)$, the lower predictive bound (LPB) for a test point $X_\textup{test}=x$.
\end{algorithmic}
\end{algorithm}

\subsection{Calibration With a Known Censoring Mechanism: \texttt{Trimmed} Mode}\label{sec:calibration_trimmed}

We describe the calibration procedure with trimmed budget allocation in Algorithm~\ref{alg:calibration_trimmed}.

\begin{algorithm}
\caption{Calibration With a Known Censoring Mechanism: \texttt{Trimmed} Mode}
\label{alg:calibration_trimmed}
\begin{algorithmic}[1]
\REQUIRE Calibration data $\{X_i\}_{i\in\Itwo}$, generative model \(\model(\cdot)\), audit function \(\texttt{Audit}(\cdot)\), pre-trained quantile regression model $\{\hat{q}_\tau(\cdot)\}_{\tau\in\mathcal{T}}$, target miscoverage rate $\tautarget$, prior quantile $\tauprior$, total budget $B$.
\\\medskip

    \STATE $\{\hat{f}_\tau(X_i)\}_{\tau\in \mathcal{T},i \in \Itwo} \gets \{\hat{q}_\tau(X_i)\}_{\tau\in \mathcal{T},i \in \Itwo}$
    \STATE $\{\cprob_i\}\gets \{\min(\tfrac{B}{|\Itwo|\;\hat f_\tauprior(X_i)},1)\}$ {\color{darkgray} \ \ // per-prompt evaluation probability (Section~\ref{subsec:prompt_adaptive})} \medskip

\STATE $\{(\tilde{T}_i, C_i)\}_{i\in\Itwo} \gets $ Algorithm~\ref{alg:sample_calibration_set} applied to $\{X_i\}_{i\in\Itwo}$ with $\{\cprob_i\}_{i\in\Itwo}$ and $\{\hat{f}_\tauprior(X_i)\}_{i\in\Itwo}$
\STATE $w_\tau(\{X_j\}_{j\in \Itwo}, i) \gets \frac{1}{\cprob_i}, \ i\in\Itwo$  

\STATE Obtain $\hat{L}(x)$ from Algorithm~\ref{alg:calibration_with_known_censoring}, applied with $\{X_i, C_i, \tilde{T}_i\}_{i\in\Itwo}$, $\{w_\tau(\{X_j\}_{j\in \Itwo}, i)\}_{i\in\Itwo}$, $\{\hat{q}_\tau(\cdot)\}_{\tau\in\mathcal{T}}$, $\tautarget$, and $\tauprior$.

\RETURN $\hat{L}(x)$, the lower predictive bound (LPB) for a test point $X_\textup{test}=x$.
\end{algorithmic}
\end{algorithm}

\subsection{Calibration With a Known Censoring Mechanism: \texttt{Optimized} Mode}\label{sec:optimized_alg}

In this section, we detail the optimized adaptive budgeting algorithm.
First, observe that if the budget is sufficiently large such that the sum of $\hat{f}_{\tauprior}(X_i)$ does not exceed the budget, i.e., $\sum_{i\in\mathcal I_2}\hat{f}_{\tauprior}(X_i) \leq B$, then we can set the censoring times to the target value $C_i=\hat{f}_{\tauprior}(X_i)$ for all calibration points. In this case, the solution for~\eqref{eq:global‐opt} is straightforward, and the optimized probabilities are equal to one: $\pi_i^* = 1, \  \forall i \in\Itwo$. Since this setup is trivial, we will consider the more challenging setting, in which the target values exceed the threshold $\sum_{i\in\mathcal I_2}\hat{f}_{\tauprior}(X_i) > B$.

Since each summand $1/\pi_i$ is strictly convex in $\pi_i$ on $(0,1]$ and the constraint is linear, \eqref{eq:global‐opt} is a convex program. Furthermore, if the sum of $\hat{f}_{\tauprior}(X_i)$ exceeds the budget, the constraint turns into an equality constraint, as stated next.
\begin{proposition}\label{prop:budget‐tight}
Assume $\sum_{i\in\mathcal I_2}\hat{f}_{\tauprior}(X_i) > B$.  Then any minimizer $\pi^*$ of~\eqref{eq:global‐opt} must satisfy the budget constraint with equality: $$\sum_{i\in\mathcal I_2}\hat{f}_{\tauprior}(X_i)\,\pi^*_i \;=\; B.$$
\end{proposition}
We now turn to show how to compute the optimized probabilities $\pi^*$. Under Proposition~\ref{prop:budget‐tight}, the optimization constraint in~\eqref{eq:global‐opt} turns into an equality constraint, and the Lagrangian is therefore given by
\[
\mathcal L(\pi,\lambda)
=\frac{1}{|\Itwo|}\sum_{i\in\mathcal I_2}\frac{1}{\pi_i}
+\lambda\Bigl(\sum_{i\in\mathcal I_2}\hat{f}_{\tauprior}(X_i)\,\pi_i - B\Bigr),
\]
where $\lambda>0$ is the Lagrange multiplier.  On the open domain $(0,1]^n$, the objective is strictly convex and differentiable, so the single stationary point satisfying the first‐order conditions is in fact the unique global minimizer. Taking partial derivatives,
\[
\frac{\partial\mathcal L}{\partial\pi_i}
=-\frac{1}{|\Itwo|\,\pi_i^2}+\lambda\,\hat{f}_{\tauprior}(X_i)
\;=\;0
\quad\Longrightarrow\quad
\pi_i=\frac{1}{\sqrt{|\Itwo|\,\lambda\,\hat{f}_{\tauprior}(X_i)}}.
\]
Applying the box constraint $\pi_i\le1$ gives
\[
\pi_i^*(\lambda)
=\min \Bigl\{1,\;\frac{1}{\sqrt{|\Itwo|\,\lambda\,\hat{f}_{\tauprior}(X_i)}}\Bigr\}.
\]
Now, define the budget‐usage function
\[
U(\lambda)\;=\;\sum_{i\in\mathcal I_2}\hat{f}_{\tauprior}(X_i)\,\pi_i^*(\lambda).
\]
The above equation has a unique solution $\lambda^* >0$ which satisfies $U(\lambda^*)=B$. This $\lambda^* >0$ can be recovered by bisection, across the following interval:
\begin{equation}
\left[\lambda_{\textup{low}}=\;\frac{1}{|\Itwo|\max_{i\in\mathcal I_2}\hat{f}_{\tauprior}(X_i)}, \ \lambda_{\textup{high}}=\frac{|\Itwo| \max_{i\in\mathcal I_2}\hat{f}_{\tauprior}(X_i)^2}{B^2\min_{i\in\mathcal I_2}\hat{f}_{\tauprior}(X_i)}\right].
\end{equation}
This procedure is summarized in Algorithm~\ref{alg:SolveOptimization} and it produces a valid solution, as stated in the following proposition.
\begin{proposition}\label{prop:unique_lambda}
The output $\pi^*$ of Algorithm~\ref{alg:SolveOptimization} is the unique solution of~\eqref{eq:global‐opt} which satisfies $\sum_{i\in\mathcal I_2}\hat{f}_{\tauprior}(X_i)\,\pi^*_i \!\le\! B$.
\end{proposition}

\begin{algorithm}
\caption{Calibration with a known censoring mechanism: optimized budget allocation}
\label{alg:calibration_optimized}
\begin{algorithmic}[1]
\REQUIRE Calibration data $\{X_i\}_{i\in\Itwo}$, generative model \(\model(\cdot)\), audit function \(\texttt{Audit}(\cdot)\), pre-trained quantile regression model $\{\hat{q}_\tau(\cdot)\}_{\tau\in\mathcal{T}}$, target miscoverage rate $\tautarget$, prior quantile $\tauprior$, quantile trimming threshold $M$, total budget $B$.
\\\medskip

    \STATE $\{\hat{f}_\tau(X_i)\}_{\tau\in \mathcal{T},i \in \Itwo} \gets \{\min{(\hat{q}_\tau(X_i),M)}\}_{\tau\in \mathcal{T},i \in \Itwo}$ {\color{darkgray} \ \ // trim the quantile est. (Section~\ref{subsec:prompt_capped})}
    \STATE $\{\cprob_i\}\gets$ Algorithm~\ref{alg:SolveOptimization} applied with $\{\hat{f}_\tauprior(X_i)\}_{i \in \Itwo}$ {\color{darkgray} \ \ // optimized per-prompt evaluation \\\qquad\qquad\qquad\qquad\qquad\qquad\qquad\qquad\qquad\qquad\qquad probability (Section~\ref{subsec:prompt_global})}
\medskip

\STATE $\{(\tilde{T}_i, C_i)\}_{i\in\Itwo} \gets $ Algorithm~\ref{alg:sample_calibration_set} applied to $\{X_i\}_{i\in\Itwo}$ with $\{\cprob_i\}_{i\in\Itwo}$ and $\{\hat{f}_\tauprior(X_i)\}_{i\in\Itwo}$
\STATE $w(\{X_j\}_{j\in\Itwo},i) \gets \frac{1}{\cprob_i}, \ i\in\Itwo$  

\STATE Obtain $\hat{L}(x)$ from Algorithm~\ref{alg:calibration_with_known_censoring}, applied with $\{X_i, C_i, \tilde{T}_i\}_{i\in\Itwo}$, $\{w_{i,}\}_{i\in\Itwo}$, $\{\hat{f}_\tau(\cdot)\}_{\tau\in\mathcal{T}}$, $\tautarget$, and $\tauprior$.

\RETURN $\hat{L}(x)$, the lower predictive bound (LPB) for a test point $X_\textup{test}=x$.
\end{algorithmic}
\end{algorithm}

\begin{algorithm}[H]
\caption{Get Optimal Per-Prompt Evaluation Probabilities}
\label{alg:SolveOptimization}
\begin{algorithmic}[1]
\REQUIRE effective sample cost \(\{\hat{f}_{\tauprior}(X_i)\}_{i \in \mathcal{I}_2}\), budget \(B\), bijection tolerance \(\epsilon\), initial \(\lambda_{\textup low},\lambda_{\textup high}\)
\ENSURE Optimal \(\pi^*\) and multiplier \(\lambda^*\) satisfying
\(\pi^*_i=\min\{1,1/\sqrt{\lambda^*\,\hat{f}_{\tauprior}(X_i)}\}\) and \(\sum_{i\in \Itwo} \hat{f}_{\tauprior}(X_i)\,\pi^*_i=B\).
\IF{\(B>\sum_{i\in \Itwo} \hat{f}_{\tauprior}(X_i)\)}
    \STATE \(\pi^*\gets[1,\dots,1]\), \(\lambda^*\gets\texttt{Null}\)
\ELSE
    \WHILE{\(\sum_{i\in \Itwo} \hat{f}_{\tauprior}(X_i)\,\pi_i^*(\lambda_{\textup high}) > B\)}
        \STATE \(\lambda_{\textup{high}}\gets 2\,\lambda_{\textup{high}}\)
    \ENDWHILE
    \STATE \(\lambda^*\gets\) bisection on \([\lambda_{\textup{low}},\lambda_{\textup{high}}]\) to solve 
    \(\sum_{i\in \Itwo} \hat{f}_{\tauprior}(X_i)\,\pi_i^*(\lambda)=B\) within \(\epsilon\)
    \STATE \(\pi^*_i\gets \min\{1,1/\sqrt{\lambda^*\,\hat{f}_{\tauprior}(X_i)}\}\) for all \(i\in \Itwo\)
\ENDIF
\RETURN \((\pi^*,\lambda^*)\)
\end{algorithmic}
\end{algorithm}

\section{A Deeper Look Into \texorpdfstring{$\hat{\alpha}(\tau)$}{alpha(tau)}}
\label{appendix:alpha_hat}

In Section~\ref{sec:calibration}, we introduced the miscoverage estimator $\hat{\alpha}(\tau)$. In this section, we explore its statistical properties in detail. We begin by deriving its expectation and variance:

\begin{proposition}[Unbiasedness and conditional variance of the weighted miscoverage estimator]
\label{prop:unbiasedness_variance_conditional}
Under the conditional independence assumption~\ref{assumption:conditionally_independent_censoring} in Section~\ref{sec:setup}, the estimator $\hat\alpha(\tau)$ defined in equation~\eqref{eq:miscoverage_estimator} satisfies
\[
\mathbb{E}\bigl[\hat\alpha(\tau)\bigr]
=\frac{1}{|\mathcal I_2|}\sum_{i\in\mathcal I_2}\mathbb{P}\bigl[T_i<\hat q_\tau(X_i)\bigr].
\]
If we further assume that for all $i\neq i'\in\Itwo$, $(C_i,T_i) \indep(C_{i'},T_{i'})\mid\{X_j\}_{j\in\mathcal I_2}$, we have 
\begin{align*}
\Var\bigl[\hat\alpha(\tau)\,\bigm|\,\{X_j\}_{j\in\mathcal I_2}\bigr]
= \frac{1}{|\mathcal I_2|^2}
   \sum_{i\in\mathcal I_2}
   \Bigl\{
     w(\{X_j\}_{j\in \Itwo}, i)&\,
     \mathbb{P}\bigl[T_i<\hat q_\tau(X_i)\,\bigm|\;\{X_j\}_{j\in\mathcal I_2}\bigr]
     \\[-0.3em]
    &-\;
     \bigl[\mathbb{P}\bigl[T_i<\hat q_\tau(X_i)\,\bigm|\;\{X_j\}_{j\in\mathcal I_2}\bigr]\bigr]^2
   \Bigr\}.
\end{align*}
\end{proposition} 
See Appendix~\ref{appendix:proof_unbiasedness} for the proof. We remark that the conditional independence and pair-wise independence assumptions in Proposition~\ref{prop:unbiasedness_variance_conditional} hold for all variants of our method from Sections~\ref{sec:methods}. This proposition shows that $\hat{\alpha}(\tau)$ is an unbiased miscoverage estimator, resulting in a useful approximation for the miscoverage estimation variance. Under a simplified assumption that, for a fixed $\tau$, each calibration point $i\in\Itwo$ is miscovered at the same constant rate conditional on $\{X_j\}_{j\in\mathcal I_2}$, we have the following result.
\begin{proposition}[Variance linearly increasing in mean weight under constant miscoverage]
\label{prop:variance_constant_miscoverage}
Under Assumption~\ref{assumption:conditionally_independent_censoring}, and additionally assuming that for all $i\neq i'$, $(C_i,T_i)\indep(C_{i'},T_{i'})\mid\{X_j\}_{j\in\mathcal I_2}$, and that for any fixed $\tau$, each calibration point $i\in\Itwo$ is miscovered at the same constant rate conditional on $\{X_j\}_{j\in\mathcal I_2}$,
\[
\mathbb{P}\bigl[T_i<\hat q_\tau(X_i)\,\bigm|\;\{X_j\}_{j\in\mathcal I_2}\bigr] = \mathrm{Const}_\tau.
\]
Then, the variance of $\hat\alpha(\tau)$ from~\eqref{eq:miscoverage_estimator} is
\begin{align*}
\Var\bigl[\hat\alpha(\tau)\,\bigm|\,\{X_j\}_{j\in\mathcal I_2}\bigr]
&=\frac{\mathrm{Const}_\tau}{|\mathcal I_2|}\overline w_\tau-\frac{\mathrm{Const}^2_\tau}{|\mathcal I_2|},
\end{align*}
where $\overline w_\tau=|\mathcal I_2|^{-1}\sum_{i\in\mathcal I_2}w_\tau(\{X_j\}_{j\in\mathcal I_2},i)$ is the mean weight.  
\end{proposition}
See proof in Appendix~\ref{appendix:proof_variance_constant_miscoverage}.

\section{Proofs}\label{sec:proofs}

\subsection{Proof of Proposition~\ref{prop:unbiasedness_variance_conditional}}
\label{appendix:proof_unbiasedness}

\begin{proof}
Recall that
\[
\hat\alpha(\tau)
=\frac{1}{|\mathcal I_2|}\sum_{i\in\mathcal I_2}
w(X_i)\,
\1\{\hat q_\tau(X_i)\le C_i\}\,
\1\{T_i<\hat q_\tau(X_i)\}.
\]Set
\[
V_i=w(X_i)\,
\1\{\hat q_\tau(X_i)\le C_i\}\,
\1\{T_i<\hat q_\tau(X_i)\}.
\]
By Assumption~\ref{assumption:conditionally_independent_censoring} each of \(C_i\) and \(T_i\) is independent of the other given \(X_i\).  Hence for each \(i\),
\begin{align*}
\E\bigl[V_i \mid \{X_j\}_{j\in\Itwo}\bigr]
&=\E \bigl[w(\{X_j\}_{j\in \Itwo}, i)\,
  \1\{\hat q_\tau(X_i)\le C_i\}\,
\1\{T_i<\hat q_\tau(X_i)\}\,\bigm|\;\{X_j\}_{j\in\Itwo}\bigr]\\
&=\E \bigl[w(\{X_j\}_{j\in \Itwo}, i)\,
  \1\{\hat q_\tau(X_i)\le C_i\}\,\bigm|\;\{X_j\}_{j\in\Itwo}\bigr]\,
\E \bigl[\1\{T_i<\hat q_\tau(X_i)\}\,\bigm|\;\{X_j\}_{j\in\Itwo}\bigr]\\
&=w(\{X_j\}_{j\in \Itwo}, i)\,
  \mathbb{P}\bigl[\hat q_\tau(X_i)\le C_i\,\bigm|\{X_j\}_{j\in\Itwo}\bigr]\,\mathbb{P} \bigl[T_i<\hat q_\tau(X_i)\,\bigm|\;\{X_j\}_{j\in\Itwo}\bigr]\\
&=\mathbb{P}\bigl[T_i<\hat q_\tau(X_i)\,\bigm|\;\{X_j\}_{j\in\Itwo}\bigr],
\end{align*}
where the first is by the tower property, and in the last transition we used 
\(
w(\{X_j\}_{j\in \Itwo}, i)=\mathbb P\bigl[\hat q_\tau(X_i)\le C_i\bigm| \{X_j\}_{j\in \Itwo}\bigr]^{-1},
\)
from~\eqref{eq:weight_def}.

Consequently,
\begin{align*}
\mathbb{E}[\hat\alpha(\tau)]
&=\mathbb{E}\Bigl[\,
  \tfrac1{|\mathcal I_2|}\sum_{i\in\mathcal I_2}
  V_i
  \Bigr]\\
&=\mathbb{E}\Bigl[\,
  \tfrac1{|\mathcal I_2|}\sum_{i\in\mathcal I_2}
  \E\bigl[V_i \mid \{X_j\}_{j\in\Itwo}\bigr]
  \Bigr]\\
&=\mathbb{E}\Bigl[\,
  \tfrac1{|\mathcal I_2|}\sum_{i\in\mathcal I_2}
  \mathbb{P}\bigl[T_i<\hat q_\tau(X_i)\mid X_i\bigr]
  \Bigr]
=\frac{1}{|\mathcal I_2|}\sum_{i\in\mathcal I_2}
\mathbb{P}\bigl[T_i<\hat q_\tau(X_i)\bigr].
\end{align*} 

For the conditional variance,
we similarly derive
\begin{align*}
\E\bigl[V_i^2 \mid \{X_j\}_{j\in\Itwo}\bigr]
&=\E \bigl[w^2_\tau(\{X_j\}_{j\in \Itwo}, i)\,
  \1\{\hat q_\tau(X_i)\le C_i\}\,
\1\{T_i<\hat q_\tau(X_i)\}\,\bigm|\;\{X_j\}_{j\in\Itwo}\bigr]\\
&=w^2_\tau(\{X_j\}_{j\in \Itwo}, i)\,
  \mathbb{P}\bigl[\hat q_\tau(X_i)\le C_i\,\bigm|\{X_j\}_{j\in\Itwo}\bigr]\,\mathbb{P} \bigl[T_i<\hat q_\tau(X_i)\,\bigm|\;\{X_j\}_{j\in\Itwo}\bigr]\\
&=w_\tau(\{X_j\}_{j\in \Itwo}, i)\mathbb{P}\bigl[T_i<\hat q_\tau(X_i)\,\bigm|\;\{X_j\}_{j\in\Itwo}\bigr].
\end{align*}
Therefore
\begin{align*}
\Var\bigl[V_i\,\bigm|\;\{X_j\}_{j\in\Itwo}\bigr]
=&\E\bigl[V_i^2\,\bigm|\;\{X_j\}_{j\in\Itwo}\bigr]
-\bigl(\E[V_i\,\bigm|\;\{X_j\}_{j\in\Itwo}]\bigr)^2\\
=&\,w(\{X_j\}_{j\in \Itwo}, i)\,
  \mathbb{P}\bigl[T_i<\hat q_\tau(X_i)\,\bigm|\;\{X_j\}\bigr]
-\Bigl[\mathbb{P}\bigl[T_i<\hat q_\tau(X_i)\,\bigm|\;\{X_j\}\bigr]\Bigr]^2.
\end{align*}
Conditioned on \(\{X_j\}_{j\in\mathcal I_2}\), the pairs \((C_i,T_i)\) are independent across \(i\), and so for \(i\neq j\), \[\Cov(V_i,V_j\mid\{X_j\})=0.\] 
Consequently, 
\begin{align*}
\Var\bigl[\hat{\alpha}(\tau)\,\bigm|\;\{X_j\}_{j\in\Itwo}\bigr]=&\Var\Bigl[\tfrac1{|\mathcal I_2|}\sum_{i\in\mathcal I_2}V_i\,\Bigm|\;\{X_j\}_{j\in\Itwo}\Bigr]\\=&\tfrac1{|\mathcal I_2|^2}\sum_{i\in\mathcal I_2}\Bigl\{
     w(\{X_j\}_{j\in \Itwo}, i)\,
     \mathbb{P}\bigl[T_i<\hat q_\tau(X_i)\,\bigm|\;\{X_j\}_{j\in\mathcal I_2}\bigr]
     \\[-0.3em]
    &\qquad\qquad\qquad\qquad-\;
     \bigl[\mathbb{P}\bigl[T_i<\hat q_\tau(X_i)\,\bigm|\;\{X_j\}_{j\in\mathcal I_2}\bigr]\bigr]^2
   \Bigr\}
\end{align*}
\end{proof}

\subsection{Proof of Proposition~\ref{prop:variance_constant_miscoverage}}
\label{appendix:proof_variance_constant_miscoverage}
\begin{proof}[Proof of Proposition~\ref{prop:variance_constant_miscoverage}]
Under the proposition's assumptions, Proposition~\ref{prop:unbiasedness_variance_conditional} shows us that the variance of $\hat\alpha(\tau)$ simplifies to
\begin{align*}
\Var\bigl[\hat\alpha(\tau)\,\bigm|\,\{X_j\}_{j\in\mathcal I_2}\bigr]
&= \frac{1}{|\mathcal I_2|^2}
   \sum_{i\in\mathcal I_2}
   \Bigl\{
     w_\tau(\{X_j\}_{j\in \Itwo}, i)\,
     \mathrm{Const}_\tau
     -
     \mathrm{Const}_\tau^2
   \Bigr\}\\
&=\frac{\mathrm{Const}_\tau}{|\mathcal I_2|}\overline w_\tau-\frac{\mathrm{Const}^2_\tau}{|\mathcal I_2|}
\end{align*}
where $\overline w_\tau=|\mathcal I_2|^{-1}\sum_{i\in\mathcal I_2}w_\tau(\{X_j\}_{j\in\mathcal I_2},i)$ is the mean weight.  
\end{proof}

\subsection{Proof of Proposition~\ref{prop:budget‐tight}}

\begin{proof}[Proof of Proposition~\ref{prop:budget‐tight}]
Since
\[
\frac{\partial}{\partial\pi_i}\Bigl(\frac{1}{|\Itwo|}\sum_{j\in \Itwo}\tfrac1{\pi_j}\Bigr)
=-\frac1{|\Itwo|\,\pi_i^2}<0,
\]
increasing any single $\pi_i$ (while keeping the others fixed) strictly decreases the objective. Assume for the sake of contradiction that there exist an optimal solution for~\eqref{eq:global‐opt}, $\pi^*$, satisfying $\sum_{i\in \Itwo} \hat{f}_{\tauprior}(X_i)\,\pi^*_i < B$. Then, there exist an index $j \in \Itwo$ and $\delta>0$ such that
$\pi^*_j+\delta\le1$ and
$\sum_i \hat{f}_{\tauprior}(X_i)\,\tilde\pi_i\le B$, where
$\tilde\pi_i=\pi^*_i$ for $i\neq j$ and $\tilde\pi_j=\pi^*_j+\delta$.
But then
\[
\frac{1}{|\Itwo|}\sum_{i\in \Itwo} \frac1{\tilde\pi_i}
\;<\;\frac{1}{|\Itwo|}\sum_{i\in \Itwo} \frac1{\pi^*_i},
\]
contradicting optimality of $\pi^*$. Therefore, the budget must bind:
\[
\sum_{i\in\mathcal I_2}\hat{f}_{\tauprior}(X_i)\,\pi^*_i = B.
\]
\end{proof}

\subsection{Proof of Proposition~\ref{prop:unique_lambda}}

\begin{proof}[Proof of Proposition~\ref{prop:unique_lambda}]
Notice that if $\sum_{i\in\mathcal I_2}\hat{f}_{\tauprior}(X_i) \leq B$, then $\pi_i^* =1$ for all $i \in \Itwo$ is the optimal solution which satisfies the constraint. We now turn to consider the case where $\sum_{i\in\mathcal I_2}\hat{f}_{\tauprior}(X_i) > B$.
We first show that there exists a unique solution $\lambda^*$. Observe that $U:(0,\infty)\to(0,\sum_i \hat{f}_{\tauprior}(X_i)]$ is continuous and strictly decreasing, with
\(\lim_{\lambda\to0^+}U(\lambda)=\sum_i \hat{f}_{\tauprior}(X_i)>B\) and \(\lim_{\lambda\to\infty}U(\lambda)=0\).
By the intermediate‐value theorem there is a unique \(\lambda^*>0\) such that \(U(\lambda^*)=B\). 
Next, the output of Algorithm~\ref{alg:SolveOptimization} is indeed the solution $\lambda^*$ since $U$ is a monotone function, and thus the bisection approach is valid in this case. Therefore, $\pi^*(\lambda)$ is the correct optimal solution.
\end{proof}

\subsection{Proof of Proposition~\ref{prop:min_prob_bound}}



\begin{proof}[Proof of Proposition~\ref{prop:min_prob_bound}]
First, notice that by assuming $\sum_{i\in\mathcal I_2}\hat{f}_{\tauprior}(X_i) \leq B$, is follows that $\pi_i^* =1$ for all $i \in \Itwo$ is the optimal solution which satisfies the constraint. We now turn to analyze the setup where $\sum_{i\in\mathcal I_2}\hat{f}_{\tauprior}(X_i) > B$.
We prove by induction over the sum $\sum_{i \in \Itwo} \hat{f}_{\tauprior}(X_i)$, assuming that $\hat{f}_{\tauprior}(X_i) \leq M \ \forall i \in \Itwo$.

\textbf{Base case.}
We begin with the maximal sum: $\sum_{i \in \Itwo} \hat{f}_{\tauprior}(X_i)= |\Itwo| M$. That is, $\hat{f}_{\tauprior}(X_i)=M$ for all $i\in\mathcal I_2$. Then, by the optimization constraint, we get
\[
M\sum_{i\in\mathcal I_2}\pi^*_i = B
\qquad\Longrightarrow\qquad
\sum_{i\in\mathcal I_2}\pi^*_i = \frac{B}{M},
\]
and by symmetry the unique minimizer of $\sum_i 1/\pi^*_i$ with $\sum_i\pi^*_i=B/M$ is
\[
\pi^*_i = \frac{B}{n\,M}\quad\forall i\in\Itwo.
\]

\textbf{Inductive step.}
We now suppose that the claim holds for any weight vector with a sum $B-1< S \leq M |\Itwo|$, and show that it holds for any vector with a sum $S-1$.
Given a weight vector $\{\hat{f}_{\tauprior}(X_i)\}_{i \in \Itwo}$ with a sum $\sum_{i \in \Itwo} \hat{f}_{\tauprior}(X_i)= S-1$, we compose a new vector with a sum $S$, as follows. In $\{\hat{f}_{\tauprior}(X_i)\}_{i \in \Itwo}$, there exists an index $i_0$ such that $\hat{f}_{\tauprior}(X_{i_0}) \leq M-1$. The new vector with a sum $S$ is defined by:
\[
\hat{f}'_{\tauprior}(X_i) = 
\begin{cases}
\hat{f}_{\tauprior}(X_i), & i\neq i_0,\\
\hat{f}_{\tauprior}(X_i)+1, & i=i_0,
\end{cases}
\]
Notice that $\max_i \hat{f}'_{\tauprior}(X_i) \leq M$ and $\sum_{i \in \Itwo} \hat{f}'_{\tauprior}(X_i)= S$. Let $\lambda$ and $\lambda'$ be the unique solutions of
\[
U(\lambda)=B
\quad\text{and}\quad
U'(\lambda')=B,
\]
where
\[
U'(\lambda)
=\sum_{i\neq i_0}\hat{f}_{\tauprior}(X_i)\,\min\!\Bigl\{1,\tfrac{1}{\sqrt{n\lambda \hat{f}_{\tauprior}(X_i)}}\Bigr\}
+(\hat{f}_{\tauprior}(X_{i_0})+1)\,\min\!\Bigl\{1,\tfrac{1}{\sqrt{n\lambda(\hat{f}_{\tauprior}(X_{i_0})+1)}}\Bigr\}.
\]
Since for every $\lambda>0$ we have 
\[
\frac{\hat{f}_{\tauprior}(X_{i_0})+1}{\sqrt{|\Itwo|\lambda(\hat{f}_{\tauprior}(X_{i_0})+1)}}=\frac{\sqrt{\hat{f}_{\tauprior}(X_{i_0})+1}}{\sqrt{|\Itwo|\lambda}} > \frac{\sqrt{\hat{f}_{\tauprior}(X_{i_0})}}{\sqrt{|\Itwo|\lambda}}=\frac{\hat{f}_{\tauprior}(X_{i_0})}{\sqrt{|\Itwo|\lambda\hat{f}_{\tauprior}(X_{i_0})}},
\]
and of course $\hat{f}_{\tauprior}(X_{i_0})+1>\hat{f}_{\tauprior}(X_{i_0})$, we get the inequality
\[
(\hat{f}_{\tauprior}(X_{i_0})+1)\,\min\!\Bigl\{1,\tfrac{1}{\sqrt{|\Itwo|\lambda(\hat{f}_{\tauprior}(X_{i_0})+1)}}\Bigr\}>\hat{f}_{\tauprior}(X_{i_0})\,\min\!\Bigl\{1,\tfrac{1}{\sqrt{|\Itwo|\lambda\hat{f}_{\tauprior}(X_{i_0})}}\Bigr\},
\]
and thus,
\begin{align*}
    U'(\lambda) =&\sum_{i\neq i_0}\hat{f}_{\tauprior}(X_i)\,\pi_i^*(\lambda) +(\hat{f}_{\tauprior}(X_{i_0})+1)\,\min\!\Bigl\{1,\tfrac{1}{\sqrt{|\Itwo|\lambda(\hat{f}_{\tauprior}(X_{i_0})+1)}}\Bigr\}\\ 
    >& \sum_{i\neq i_0}\hat{f}_{\tauprior}(X_i)\,\pi_i^*(\lambda) +\hat{f}_{\tauprior}(X_{i_0})\,\min\!\Bigl\{1,\tfrac{1}{\sqrt{|\Itwo|\lambda\hat{f}_{\tauprior}(X_{i_0})}}\Bigr\}\\
    =& U(\lambda) = B.
\end{align*}
Since $U'$ is continuous and decreasing on $(0,\infty)$ with 
\(\lim_{\mu\to0^+}U'(\mu)=\sum_i \hat{f}_{\tauprior}(X_i) >B\)
and \(\lim_{\mu\to\infty}U'(\mu)=0\),
the intermediate‐value theorem guarantees a unique $\lambda'$ with $U'(\lambda')=B$.  Moreover, since $U'$ is decreasing, and $U'(\lambda)\geq B=U'(\lambda')$ we get $\lambda' \geq \lambda$.
Define
\(\pi_i\) the optimal solution of~\eqref{eq:global‐opt} with $\hat{f}_{\tauprior}(X_i)$ and \(\pi'_i\) the optimal solution with $\hat{f}'_{\tauprior}(X_i)$.

- If \(i\neq i_0\), since \(\hat{f}'_{\tauprior}(X_i)=\hat{f}_{\tauprior}(X_i)\),
  \[
  \pi'_i
  =\min\!\Bigl\{1,\tfrac{1}{\sqrt{|\Itwo|\,\lambda'\,\hat{f}'_{\tauprior}(X_i)}}\Bigr\}
  \;\leq\;
  \min\!\Bigl\{1,\tfrac{1}{\sqrt{|\Itwo|\,\lambda\,\hat{f}_{\tauprior}(X_i)}}\Bigr\}
  =\pi_i.
  \]
- If \(i=i_0\), then \(\hat{f}'_{\tauprior}(X_{i_0})=\hat{f}_{\tauprior}(X_{i_0})+1>\hat{f}_{\tauprior}(X_{i_0})\), so
  \[
  \pi'_{i_0}
  =\min\!\Bigl\{1,\tfrac{1}{\sqrt{|\Itwo|\,\lambda'\,(\hat{f}_{\tauprior}(X_{i_0})+1)}}\Bigr\}
  \;\leq\;
  \min\!\Bigl\{1,\tfrac{1}{\sqrt{|\Itwo|\,\lambda'\,\hat{f}_{\tauprior}(X_{i_0})}}\Bigr\}
  \;\leq\;
  \min\!\Bigl\{1,\tfrac{1}{\sqrt{|\Itwo|\,\lambda\,\hat{f}_{\tauprior}(X_{i_0})}}\Bigr\}
  =\pi_{i_0}.
  \]
By the inductive assumption, \(\pi'_i\ge B/(|\Itwo|M)\) for all $i \in \Itwo$, hence \(\pi_i\ge B/(|\Itwo|M)\) for all $i \in \Itwo$ as well. Lastly, by the definition of the weights, we have:
\begin{equation}
w(\{X_j\}_{j\in \Itwo}, i)= ({\pi_i^*})^{-1} \leq \max \left( \frac{|\Itwo|M}{B},1 \right).
\end{equation}
This completes the proof.
\end{proof}

\subsection{Proofs of the Coverage Rate Guarantees}\label{sec:validity_proof}
In this section, we present and prove the formal version Theorem~\ref{thm:general_validity} which builds on the proof of~\citep[Theorem 3]{gui2024conformalized}.

\begin{theorem}[General validity, formal]
\label{thm:general_validity_formal}
Fix a tolerance level $\delta \in \left(0,1\right)$ and a miscoverage level $\tau\in \left(0,1\right)$. Suppose that $\{(X_i, T_i)\}_{i =1,...,n}$ and $(\Xtest, \Ttest)$ are drawn i.i.d., and that the censoring times satisfy the conditional independence assumption (Assumption~\ref{assumption:conditionally_independent_censoring}) and $ (C_i, T_i) \indep (C_j, T_j) | \{X_k\}_{k \in \Itwo} $ for all $i\ne j \in \Itwo$. We suppose that \(\hat{q}_\tau(x)\) is non-decreasing and continuous in \(\tau\). Further, assume that the weights are computed using the true probabilities:
\begin{equation}
w_\tau(\{X_j\}_{j\in\Itwo},i) = 1/{\mathbb P}\bigl[\hat q_\tau(X_i)\le C_i \bigm| \{X_j\}_{j\in\Itwo} \bigr].
\end{equation}
Above, $\hat q_\tau(X_i)$ can be either the estimated quantile of $T_i \mid X=X_i$ presented in Section~\ref{sec:related}, or its trimmed version $\hat{f}_{\tau}(X_i)$ from Section~\ref{subsec:prompt_capped}.
We assume that there exists a constant \({\gamma}_\tau>0\) such that the weights satisfy \({w}_\tau(x)\leq {\gamma}_{\tau}\) for $P_X$-almost all \(x\). Consider the following estimated miscoverage rate:
\begin{equation}
    \hat{\alpha}(\tau) = \frac{1}{|\Itwo|}\sum_{i\in\mathcal I_2} w_\tau(\{X_j\}_{j\in \Itwo}, i)\;
\mathbb{I}\bigl\{\tilde{T}_i<\hat q_\tau(X_i) \le C_i\bigr\}
\end{equation}
and denote the calibrated quantile level by 
\begin{equation}
\hat{\tau} = \sup\Bigl\{\tau \in \mathcal{T} : \sup_{\substack{\tau' \in \mathcal{T} \\ \tau' \le \tau}} \hat{\alpha}(\tau') \le \tautarget\Bigr\}.
\end{equation} 
Above, the search space $\mathcal{T}$ is formulated as in Section~\ref{sec:calibration}, or using $\mathcal{T} \cap [0,\tauprior]$, as in the proposed \texttt{Adaptive} method from Section~\ref{subsec:prompt_adaptive}.
We remark that for $\hat{\tau}$ to be well defined, we assume there exists $\tau' \in \mathcal{T}$ such that $\hat{\alpha}(\tau') \leq \alpha$ . This assumption can be trivially satisfied by setting $\hat{q}_{0}(X_i)=0$.
Then, with probability at least $1-\delta$ over the draws of $\mathcal{D}$, the LPB $\hat{L}(x) = \hat{q}_{\hat{\tau}}(x)$ satisfies
\begin{equation}
\mathbb{P}\left[\Ttest\geq \hat{L}(\Xtest)|\mathcal{D}\right] \geq 1 - \alpha -  \sup_{\tau \in [0,1]} \left\{ \sqrt{ \frac{ 2\gamma^2_{\tau} + 5  }{|\Itwo|}  \cdot  \log \left(\frac{1}{\delta} \right) }\right\}.
\end{equation}
\end{theorem}

\begin{proof}
For ease of notation, we define the coverage gap by:
\begin{equation}
\Delta := \sup_{\tau \in [0,1]} \left\{ \sqrt{ \frac{ 2\gamma^2_{\lambda} + 5  }{|\Itwo|}  \cdot  \log \left(\frac{1}{\delta} \right)} \right\}.
\end{equation}
We define the oracle miscoverage level by:
\begin{equation}
    \tau(\alpha + \Delta) = \sup \left\{\lambda \in [0,1] : \mathbb{P} \left(T < \hat{q}_\lambda(X) \mid \Ione \right) \leq \alpha + \Delta \right\}.
\end{equation}
Observe that by assuming $1 - \delta \leq \mathbb{P}(\hat{\tau} \leq \tau(\alpha + \Delta) \mid \Ione)$, we get that the event $\{\hat{\tau} \leq \tau(\alpha + \Delta) \}$ holds with probability at least $1-\delta$. Under the monotonicity of $\hat{q}_\tau$, and following the the left-continuity
of $\mathbb{P}(T \geq \hat{q}_{\tau}(X) \mid \mathcal{D})$ in $\tau$ we obtain that with probability at least $1-\delta$:
\begin{equation}
\begin{split}
    &\mathbb{P}(T \geq \hat{q}_{\hat{\tau}}(X) \mid \mathcal{D})\\
     \geq & \mathbb{P}(T \geq \hat{q}_{\tau(\alpha + \Delta)}(X) \mid \mathcal{D})\\
     \geq & 1- \alpha - \Delta.
\end{split}
\end{equation}
Now, we turn to show that $1 - \delta \leq \mathbb{P}(\hat{\tau} \leq \tau(\alpha + \Delta) \mid \Ione)$. We begin by fixing $\varepsilon > 0$, and denoting $\lambda := \tau(\alpha + \Delta) + \varepsilon$. By the definition of $\hat{\alpha}(\tau)$, we get: 
\begin{equation}\label{eq:fused_w_1}
\begin{split}
    &\mathbb{P}(\hat{\alpha}(\tau(\alpha + \Delta) + \varepsilon) \leq \alpha \mid \Ione)\\
    =&\mathbb{P}(\hat{\alpha}(\lambda) \leq \alpha \mid \Ione)\\
     = & \mathbb{P}\left(\frac{1}{|\Itwo|}\sum_{i\in\mathcal I_2}
w_{\lambda}(\{X_j\}_{j\in \Itwo}, i)\;
\mathbb{I}\bigl\{\tilde{T}_i<\hat q_\tau(X_i) \le C_i\bigr\} \leq \alpha \bigg|\Ione\bigg.\right)\\
     =&\mathbb{P}\left(\frac{1}{|\Itwo|}\sum_{i\in\mathcal I_2}
w_{\lambda}(\{X_j\}_{j\in \Itwo}, i)\;
\mathbb{I}\bigl\{\hat q_\lambda(X_i)\le C_i\bigr\}\;
\mathbb{I}\bigl\{T_i<\hat q_\lambda(X_i)\bigr\} - \alpha \leq
     0 \bigg|\Ione\bigg.\right)
\end{split}
\end{equation}
Next, we apply Markov’s inequality for any $t>0$, and obtain:
\begin{equation}\label{eq:fused_w_2}
\begin{split}
\eqref{eq:fused_w_1} & \leq \mathbb{E}\left(\exp\left(\alpha - \frac{t}{|\Itwo|}\sum_{i\in\mathcal I_2}
w_{\lambda}(\{X_j\}_{j\in \Itwo}, i)\;
\mathbb{I}\bigl\{\hat q_\lambda(X_i)\le C_i\bigr\}\;
\mathbb{I}\bigl\{T_i<\hat q_\lambda(X_i)\bigr\}\right) \bigg|\Ione \bigg.\right)
\end{split}
\end{equation}
By conditioning on $\{(X_i,T_i)\}_{i \in \Itwo}$ and following the $\frac{1}{4}$ sub-gaussianity of $w_{\lambda}(\{X_j\}_{j\in \Itwo}, i)^{-1} - \mathbb{I}\{ \hat{q}_{\lambda}(X_i) \leq C_i\}$, the conditionally independent censoring assumption (Assumption~\ref{assumption:conditionally_independent_censoring}) and the bounded weights, we have:
\begin{equation}\label{eq:fused_w_5}
\begin{split}
\mathbb{E}&\left(\exp\left(\frac{t}{|\Itwo|} \cdot \sum_{i\in \Itwo}w_{\lambda}(\{X_j\}_{j\in \Itwo}, i) \mathbb{I}\bigl\{T_i<\hat q_\lambda(X_i)\bigr\} \left(w_{\lambda}(\{X_j\}_{j\in \Itwo}, i)^{-1} - \mathbb{I}\{ \hat{q}_{\lambda}(X_i) \leq C_i\} \right) \right) \bigg| \{(X_i,T_i)\}_{i \in \Itwo}, \Ione \bigg.\right)\\
 \leq & \exp\left(\frac{t^2}{8|\Itwo|^2} \cdot \sum_{i\in \Itwo}w_{\lambda}(\{X_j\}_{j\in \Itwo}, i)^2 \mathbb{I}\bigl\{T_i<\hat q_\lambda(X_i)\bigr\}^2 \right)\\
 \leq & \exp\left(\frac{t^2}{8|\Itwo|^2} \cdot \sum_{i\in \Itwo}w_{\lambda}(\{X_j\}_{j\in \Itwo}, i)^2 \right)\\
 \leq & \exp\left(\frac{ \gamma^2_{\lambda} t^2}{8|\Itwo|} \right).
\end{split}
\end{equation}

Therefore,
\begin{equation}\label{eq:fused_w_4}
\begin{split}
\eqref{eq:fused_w_2} & \leq \exp\left(\frac{ \gamma^2_{\lambda} t^2}{8|\Itwo|} \right) \cdot \mathbb{E}\left(\exp\left( \alpha - \frac{t}{|\Itwo|}\sum_{i\in\mathcal I_2}
w_{\lambda}(\{X_j\}_{j\in \Itwo}, i)\;
w_{\lambda}(\{X_j\}_{j\in \Itwo}, i)^{-1}\;
\mathbb{I}\bigl\{T_i<\hat q_\lambda(X_i)\bigr\} \right) \bigg|\Ione \bigg.\right) \\
& =\exp\left(\frac{ \gamma^2_{\lambda} t^2}{8|\Itwo|}  \right) \cdot \mathbb{E}\left(\exp\left(\alpha - \frac{t}{|\Itwo|}\sum_{i\in\mathcal I_2}
\mathbb{I}\bigl\{T_i<\hat q_\lambda(X_i)\bigr\} \right) \bigg|\Ione \bigg.\right)
\end{split}
\end{equation}

Further, defining $p_{\lambda}(\{X_j\}_{j\in \Itwo}, i):= \mathbb{P}(T_i < \hat{q}_{\lambda}(X_i) \mid \{X_j\}_{j\in \Itwo}, \Ione)$, we condition on $\{X_j\}_{j\in \Itwo}$ and get, using the $\frac{1}{4}$ subgaussianity of $p_{\lambda}(\{X_j\}_{j\in \Itwo}, i) - \mathbb{I}\bigl\{T_i<\hat q_\lambda(X_i)\bigr\}$:
\begin{equation}
\begin{split}
\mathbb{E}&\left(\exp\left(p_{\lambda}(\{X_j\}_{j\in \Itwo}, i)) - \frac{t}{|\Itwo|}\sum_{i\in\mathcal I_2}
(\mathbb{I}\bigl\{T_i<\hat q_\lambda(X_i)\bigr\} \right) \Bigg|\{X_j\}_{j\in \Itwo}, \Ione \bigg.\right) \\
\leq & \exp\left(\frac{t^2}{8|\Itwo|^2} \cdot |\Itwo| \right)\\
=&\exp\left(\frac{t^2}{8|\Itwo|} \right)
\end{split}
\end{equation}

We plug this into~\eqref{eq:fused_w_4} to obtain:
\begin{equation}\label{eq:fused_w_6}
\begin{split}
&\eqref{eq:fused_w_4} \leq\exp\left(\frac{ (\gamma^2_{\lambda}+1) t^2}{8|\Itwo|}   \right) \mathbb{E}\left(\exp\left(\frac{t}{|\Itwo|} \cdot \sum_{i\in \Itwo} (\alpha - p_{\lambda}(X_i) ) \right) \bigg|\Ione \bigg.\right) 
\end{split}
\end{equation}
We now apply the Cauchy-Schwarz inequality:
\begin{equation}\label{eq:fused_w_7}
\begin{split}
& \mathbb{E}\left(\exp\left(\frac{t}{|\Itwo|} \cdot \sum_{i\in \Itwo}( \alpha - p_{\lambda}(X_i) ) \right) \bigg|\Ione \bigg.\right) \\
\leq &  \mathbb{E}\left(\exp\left(\frac{2t}{|\Itwo|} \cdot \sum_{i\in \Itwo}( \alpha - p_{\lambda}(X_i) ) \right) \bigg|\Ione \bigg.\right)^{1/2}
\end{split}
\end{equation}
By the definition of $\tau(\alpha + \Delta)$ we have:
\begin{equation}
\mathbb{P}(T < \hat{q}_\lambda (X) \mid \Ione) \geq \alpha + \Delta.
\end{equation}
By plugging it into the above, we get:
\begin{equation}\label{eq:fused_w_8}
\begin{split}
\eqref{eq:fused_w_7} & = \mathbb{E}\left(\exp\left(\frac{2t}{|\Itwo|} \cdot \sum_{i\in \Itwo}( \alpha - p_{\lambda}(X_i) ) \right) \bigg|\Ione \bigg.\right) \\
\leq & \mathbb{E}\left(\exp\left(\frac{2t}{|\Itwo|} \cdot \sum_{i\in \Itwo}( \mathbb{P}(T < \hat{q}_\lambda (X) \mid \Ione) - \Delta - p_{\lambda}(X_i) ) \right) \bigg|\Ione \bigg.\right)\\
= & \exp(-2t\Delta) \mathbb{E}\left(\exp\left(\frac{2t}{|\Itwo|} \cdot \sum_{i\in \Itwo}( \mathbb{P}(T < \hat{q}_\lambda (X) \mid \Ione) - p_{\lambda}(X_i) ) \right) \bigg|\Ione \bigg.\right)\\
\leq & \exp\left(\frac{ t^2}{2|\Itwo|}-2t\Delta \right).
\end{split}
\end{equation}
Above, we used the $\frac{1}{4}$-sub-gaussianity of $ \mathbb{P}(T < \hat{q}_\lambda (X) \mid \Ione) - p_{\lambda}(X_i) $. By combining it all, we obtain:
\begin{equation}\label{eq:fused_w_9}
\begin{split}
\eqref{eq:fused_w_1}  \leq & \exp\left(\frac{t^2}{2|\Itwo|} - 2t\Delta + \frac{ (\gamma^2_{\lambda}+1) t^2}{8|\Itwo|} + \frac{ \gamma^2_{\lambda} t^2}{8|\Itwo|}\right)\\
=&\exp\left(\frac{ (2\gamma^2_{\lambda}+5) t^2}{8|\Itwo|} - 2t\Delta\right)
\end{split}
\end{equation}
We define
\begin{equation}
    t = \frac{(8+2\sqrt{14})|\Itwo|\Delta}{2\gamma^2_{\lambda}  + 5}
\end{equation}
and plug it into \eqref{eq:fused_w_6}:
\begin{equation}
\begin{split}
\eqref{eq:fused_w_6}  \leq & \exp\left(\frac{ (2\gamma^2_{\lambda}+5) (\frac{(8+2\sqrt{14})|\Itwo|\Delta}{2\gamma^2_{\lambda}  + 5})^2}{8|\Itwo|} - \frac{(16+4\sqrt{14})|\Itwo|\Delta}{2\gamma^2_{\lambda}  + 5}\Delta\right) \\
= & \exp\left(\frac{ (15+4\sqrt{14})|\Itwo|\Delta^2}{(2\gamma^2_{\lambda}+5) } - \frac{(16+4\sqrt{14})}{2\gamma^2_{\lambda}  + 5}|\Itwo|\Delta^2\right)  \\
= & \exp\left(-\frac{1|\Itwo|\Delta^2}{2\gamma^2_{\lambda}  + 5}\right) \\
  \leq & \exp\left( -\log \left(\frac{1}{\delta} \right)  \right) \\
    = &\delta \\
\end{split}
\end{equation}
That is, we just showed that:
\begin{equation}
    1- \delta \leq \mathbb{P}(\hat{\alpha}(\tau(\alpha + \Delta) + \varepsilon) > \alpha \mid \Ione) \leq \mathbb{P}( \hat{\tau} < \tau(\alpha + \Delta) + \varepsilon \mid \Ione).
\end{equation}
The above equation holds for every $\varepsilon > 0$, and thus by taking $\varepsilon \rightarrow 0$, and following the continuity of the probability measure, we obtain that $1 - \delta \leq \mathbb{P}(\hat{\tau} \leq \tau(\alpha + \Delta) \mid \Ione)$, as required.
\end{proof}
Since the proposed methods construct censoring times that satisfy the conditional independence assumption of Theorem~\ref{thm:general_validity_formal}, their validity follow immediately from this theorem with the corresponding definitions of the censoring times and the weights.

\section{Experimental Details}
\label{sec:details}

This section provides additional information about the experimental setup and methodology described in Section~\ref{sec:exp}. Section~\ref{subsec:synthetic_experiments} details the synthetic experiments, including data generation and model architecture. Section~\ref{subsec:real_data_experiments} presents experiments conducted on real-world data. Finally, Section~\ref{sec:training} outlines implementation details that are shared across both experimental settings.

\subsection{Synthetic Experiments}
\label{subsec:synthetic_experiments}

\paragraph{Synthetic Covariate Generation.} As described in Section~\ref{subsec:synth} of the main text, we use synthetic data to evaluate the informativity and coverage of the LPBs constructed by our calibration procedures. Each synthetic data point consists of covariates $X_i \in \mathbb{R}^{d}$ with $d=10$, and a prompt-specific unsafe-output probability $p_i$.
The dataset is designed to simulate a setting in which most prompts have a high likelihood of producing unsafe outputs, while a smaller subset is relatively safe.
Notably, our experiments show that fitting a model on this synthetic dataset results in a large prediction error, which turns the calibration step more challenging.

To this end, we sample the covariates from a Gaussian distribution:
\begin{equation}
    \label{eq:synth_x}
X_i \mid p_i \sim \mathcal{N}\left(\mu(p_i), \sigma^2 I_d\right), \quad \text{with } \sigma = 0.1.
\end{equation}
where the mean vector $\mu(p_i)\in \mathbb{R}^d$ is designed as follows. First, we construct a pool of $p_i$ values, where (i) 90\% of which are drawn uniformly in the range $\log_{10}(p_i) \sim \mathcal{U}[-4, -3]$, and (ii) the remaining 10\% are drawn uniformly in the range $\log_{10}(p_i) \sim \mathcal{U}[-6, -5]$. We then assign each $p_i$ to a data point by sampling without replacement from this pool.
Then, we define $d$ quantile levels as follows
$$
\tau_j = 0.1 + 0.8\,\frac{j-1}{d-1} \ \ \ j=1,\ldots,d.
$$
Next, for a given probability $p_i$, we compute:
$$
\tilde\mu(p_i)
=
\bigl[F^{-1}_{\mathrm{Geom}(p_i)}(\tau_1)^{1/4},\,
\ldots,\,
F^{-1}_{\mathrm{Geom}(p_i)}(\tau_d)^{1/4}\bigr]^\top,
$$
where $F^{-1}_{\mathrm{Geom}(p_i)}$ is the quantile function of the Geometric distribution with success probability $p_i$. The $1/4$ power transformation reduces the size of the often extremely high raw quantiles. Lastly, to improve numerical stability, we normalize this vector by the average magnitude across all training examples:
$$
\bar\mu
=
\frac{1}{n\,d}\sum_{i=1}^n\sum_{j=1}^d
F^{-1}_{\mathrm{Geom}(p_i)}(\tau_j)^{1/4}, \quad
\mu(p_i) = \frac{\tilde\mu(p_i)}{\bar\mu}.
$$

\paragraph{Model Architecture and Training.}
To estimate the unsafe generation probabilities, we train a neural network with four hidden layers of size 32 each. We used ReLU for non-linearity and a sigmoid for output. We optimize the BCE loss (see Section~\ref{sec:training}) using AdamW with learning rate \(10^{-4}\), weight decay \(10^{-5}\), batch size $100$, for $10$ epochs.

\paragraph{Calibration Method Parameters.}
In all experiments, unless stated otherwise, we used the prior quantile $\tau_{\mathrm{prior}}=10^{-1/4}$, with trimming threshold $M$ set so that $\gamma=\max\left({|\Itwo| \cdot M}/{B}, 1\right)=10$.

\subsubsection{Additional Experiments}
\label{sec:additional_experiments}

To confirm that our findings on synthetic data do not depend on a particular choice of calibration/test split, we repeated the experiment from Section~\ref{subsec:synth} for 20 independent splits of calibration and test sets. Specifically, in each trial, we fixed the original training set of 45,000 prompts, then randomly partitioned the remaining pool into a calibration set of size 45,000 and a test set of size 10,000. All other settings follow the experimental setup described in Section~\ref{subsec:synthetic_experiments}.

We repeated the experiment for the following values of average sampling budget per prompt $B/|I_2|$: ${10,25,50,100,200,300,600,1200}$. In all experiments, we used the same prior quantile $\tau_{\mathrm{prior}}=10^{-1/4}$, with trimming threshold $M$ set so that $\gamma=\max\left({|\Itwo| \cdot M}/{B}, 1\right)=10$, as in Section~\ref{subsec:synth}. 

Figure~\ref{fig:budget_effect_resplit} presents the performance of the methods introduced in this work for 20 data splits. Across every budget level, the performances closely match those from our original fixed‐split experiment from Figure~\ref{fig:budget_effect} in Section~\ref{subsec:synth}. 
\begin{figure}[ht]
    \centering
    \includegraphics[width=\textwidth]{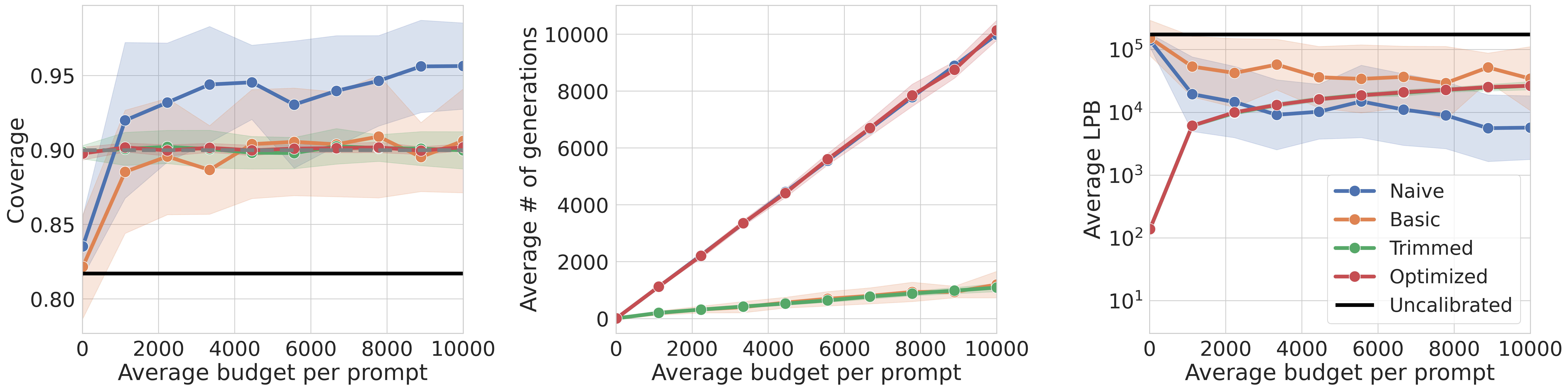}
    \caption{Results of synthetic experiments as a function of average budget per prompt \(B/|\Itwo|\). 
    {\bf Left:} Coverage (target 90\%; gray dashed line). 
    {\bf Center:} Mean number of samplings generated per prompt. 
    {\bf Right:} Mean LPB. 
    Shaded regions represent the standard deviation over 20 runs.
    \label{fig:budget_effect_resplit}
    }
\end{figure}
Next, we evaluate how the target coverage level affects the constructed LPBs and their informativeness by repeating the experiment from Section~\ref{subsec:synth} across different nominal coverage levels. We present the results in Figure~\ref{fig:budget_effect_multi_level}. This figure shows that the \texttt{Optimized} consistently achieves its target coverage across all levels.
\begin{figure}[ht]
    \centering
    \includegraphics[width=\textwidth]{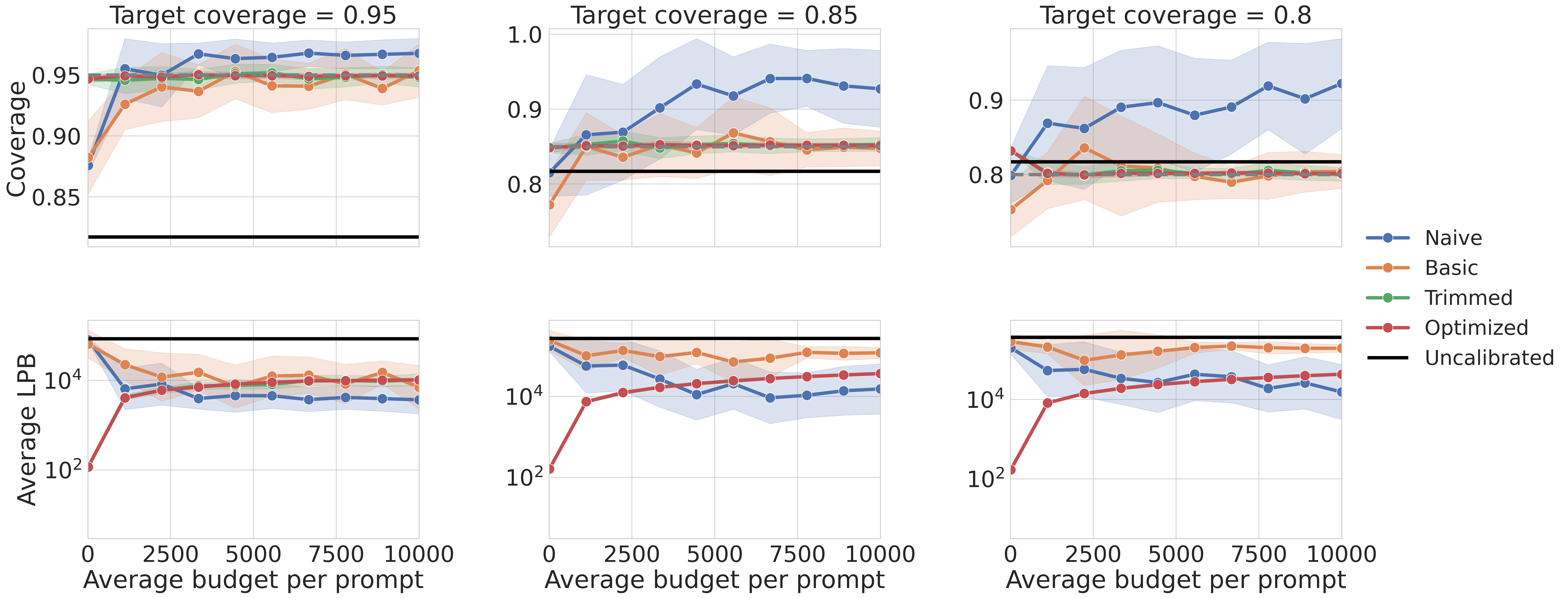}
    \caption{Results of synthetic experiments as a function of average budget per prompt \(B/|\Itwo|\). 
    {\bf Top:} Coverage (target level indicated by a gray dashed line). 
    {\bf Bottom:} Mean LPB. 
    Shaded regions represent the standard deviation over 20 runs.
    \label{fig:budget_effect_multi_level}
    }
\end{figure}

Finally, we turn to study the effect of the maximal weight size $\gamma=\max\left({|\Itwo| \cdot M}/{B}, 1\right)$ on the stability and power of the \texttt{Trimmed} and \texttt{Optimized} calibration methods. Figure~\ref{fig:min_sample_effect} compares the two methods under a fixed per‐prompt budget of ${B}/{|\Itwo|} = 1000$. For each weight size, we set the threshold $M={B \gamma}/{|\Itwo|}$. As can be seen in this figure, the size of the LPBs we construct increases with $\gamma$. However, this comes at the cost of increasing the coverage variance. Observe how the variance of \texttt{Optimized} method for $\gamma = 100$ is similar to that of \texttt{Trimmed} method for $\gamma = 10$, but the former yields much more informative LPBs under these corresponding $\gamma$ values. This is a result of the variance-reducing optimization objective in~\eqref{eq:global‐opt}.

\begin{figure}[ht]
  \centering
  \includegraphics[width=0.55\linewidth]{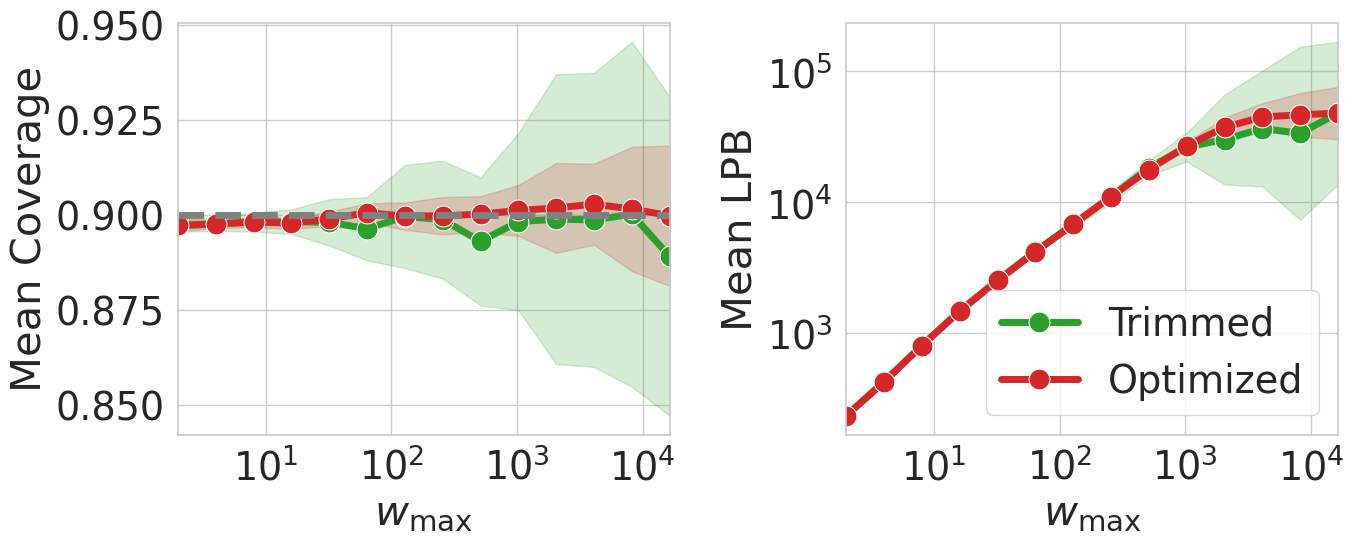}
  \caption{Results of synthetic experiments as a function of the maximum weight $w_\text{max}$ of the \texttt{Trimmed} and \texttt{Optimized} methods. \textbf{Left}: Coverage (target 90\%; gray dashed line). \textbf{Right}: Average LPB. Shaded regions present the standard deviation across $20$ runs.}
  \label{fig:min_sample_effect}
\end{figure}

We compare the proposed \texttt{Optimized} method with an oracle quantile regression model on our synthetic dataset. The oracle model outputs the true conditional quantile of $T \mid X$ as the LPB, and therefore achieves the target coverage rate of $1-\alpha = 90\%$ by definition. These oracle LPBs are also the most informative possible by construction. Figure~\ref{fig:oracle_q} reports the performance of the \texttt{Uncalibrated} quantile regression model, \texttt{Naive}, \texttt{Basic}, \texttt{Trimmed}, \texttt{Optimized}, and the oracle quantile regression model. As expected, the oracle model achieves the nominal coverage level. Notably, our \texttt{Optimized} method produces LPBs of comparable size to the oracle as the budget increases, demonstrating its efficiency in approaching the best achievable performance even without access to the oracle quantiles.
\begin{figure}[ht]
    \centering
    \includegraphics[width=\textwidth]{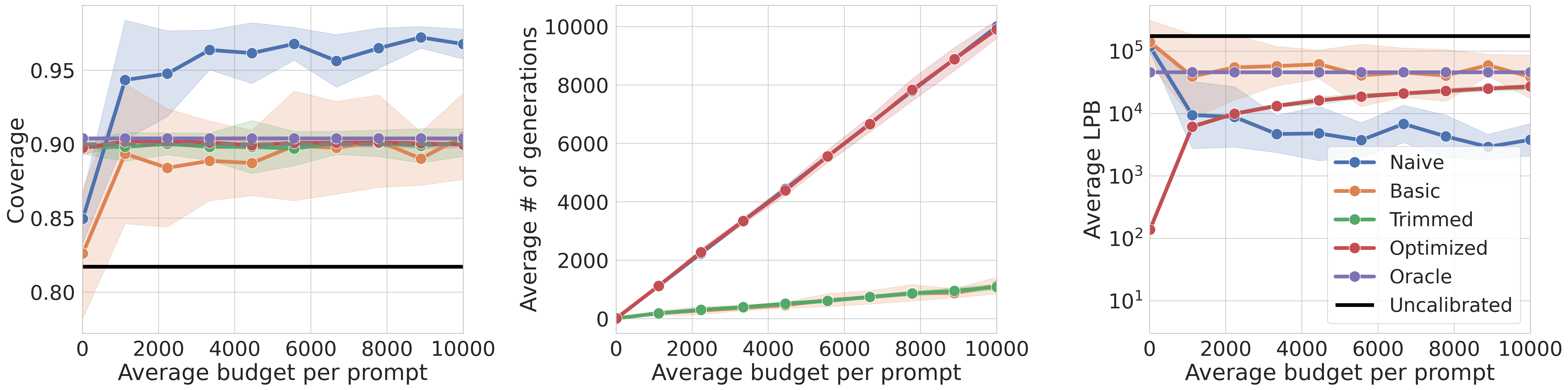}
    \caption{Synthetic data experiment with an oracle quantile regression model. 
    {\bf Top:} Coverage (target $90\%$\%; gray dashed line). 
    {\bf Bottom:} Mean LPB. 
    Shaded regions represent the standard deviation over 20 runs.
    \label{fig:oracle_q}
    }
\end{figure}

We compare our approach with a calibration method that has an infinite budget on our synthetic data. With an unlimited budget, we can observe an unsafe event for every calibration sample and thus can employ classic conformal prediction~\citep{vovk2005algorithmic, papadopoulos2008inductive} without data re-weighting to obtain LPBs that achieve $1-\alpha=90\%$ coverage rate. Specifically, we estimate the miscoverage rate similarly to~\eqref{eq:miscoverage_estimator}, but here, we use all calibration points:
\begin{align}
\label{eq:miscoverage_estimator_all}
    \hat{\alpha}^{\text{ib}}(\tau) = \frac{1}{|\Itwo|+1}\sum_{i\in\mathcal I_2}
\mathbb{I}\bigl\{T_i<\hat q_\tau(X_i) \bigr\}
\end{align}
and construct the LPB as in~\eqref{eq:our_calibration} with the above $\hat{\alpha}^{\text{ib}}$. 

We summarize the performance of all methods in Figure~\ref{fig:infinite_budget}. The results reveal that the infinite budget method achieves the desired coverage rate with lower variability than the alternatives. Moreover, our \texttt{Optimized} approach produces LPBs comparable in size to those of the infinite-budget calibration as the budget increases, demonstrating its effectiveness in approaching optimal performance even under budget constraints.

\begin{figure}[ht]
    \centering
    \includegraphics[width=\textwidth]{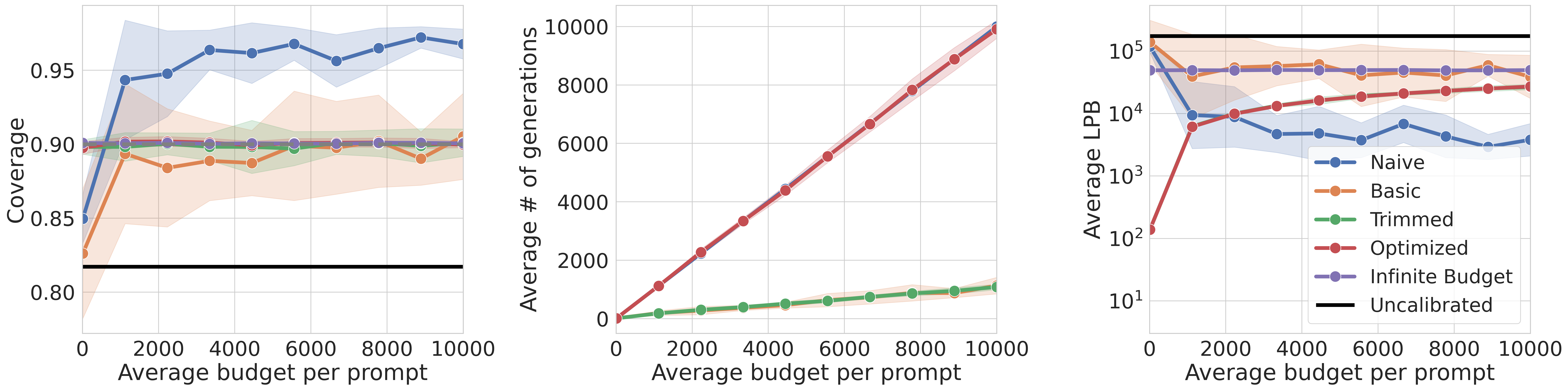}
    \caption{Synthetic data experiment with a calibration method that has an infinite budget. 
    {\bf Top:} Coverage (target 90\%; gray dashed line). 
    {\bf Bottom:} Mean LPB. 
    Shaded regions represent the standard deviation over 20 runs.
    \label{fig:infinite_budget}
    }
\end{figure}

\subsection{Real-Data Experiments}
\label{subsec:real_data_experiments}

\paragraph{Calibration Method Parameters.}
In all experiments, unless stated otherwise, we used the prior quantile $\tau_{\mathrm{prior}}=10^{-1/4}$, with trimming threshold $M$ set so that $\gamma=\max\left({|\Itwo| \cdot M}/{B}, 1\right)=2$.

\paragraph{Efficient Prompt‐Parallel Sampling.}
To efficiently generate outputs during our calibration procedures, we implement a prompt-parallel generation algorithm on top of the vLLM package~\citep{kwon2023efficient}. In each iteration, we gather a batch of $1500$ prompts ${X_i}$ along with their corresponding censoring budgets ${C_i}$, and distribute them across multiple GPUs. The pipeline iteratively performs the following steps in parallel for a batch of prompts, until all prompts in the calibration set are processed:
\begin{enumerate}
\item Generate an output $\mathcal{G}(X_i)$ using the LLM.
\item Apply the audit function $\mathrm{Audit}\bigl(X_i, \mathcal{G}(X_i)\bigr)$ to identify unsafe outputs.
\item Check whether $X_i$ has either: (i) reached the $C_i$ generations limit; or (ii) produced an output that fails the audit. If at least one of these two conditions is met, remove the prompt from the pipeline and replace it with a new prompt.
\end{enumerate}
By promptly removing any prompt that yields an unsafe output, the processing maintains high throughput and avoids idle GPU time.


\paragraph{Runtime Evaluation.}
Table~\ref{tab:runtime} presents average runtime for the \texttt{Naive} and \texttt{Optimized} calibration procedures across varying prompt budgets. While the \texttt{Naive} method can have a long runtime due to its unbounded censorship times $C_i$, it is observed to be faster than the \texttt{Optimized} method; see Table~\ref{tab:runtime}. We explain this by two key factors. First, our implementation of the \texttt{Naive} method includes an efficient runtime optimization: instead of always generating all $C_i$ samples, we terminate generation early when additional samples no longer improve the estimate of $\hat{\alpha}(\tau)$. As detailed in Section~\ref{sec:calib_impl}, this strategy yields LPBs that are equivalent to those produced by the original \texttt{Naive} procedure, while significantly reducing computation time. Second, the base model used in these experiments tends to produce low predicted quantiles $\hat{q}_\tau(X_i)$, as shown in Figure~\ref{fig:real_results}. Consequently, the effective censorship budgets $C_i$ are also low for most prompts, further limiting runtime. 

In sum, these two factors---runtime-efficient implementation and tendency to produce low predicted quantiles---explain why the reported runtime of the \texttt{Naive} is lower than the \texttt{Optimized} method in this setting.

\begin{table}[ht]
\centering
\begin{tabular}{lccc}
\toprule
Budget per prompt & Naive runtime (hours) & Optimized runtime (hours) & \# GPUs \\
\midrule
10    & 0.032 $\pm$ 0.005 & 0.063 $\pm$ 0.001 & 6 \\
25    & 0.042 $\pm$ 0.001 & 0.117 $\pm$ 0.001 & 6 \\
50    & 0.070 $\pm$ 0.006 & 0.209 $\pm$ 0.001 & 6 \\
100   & 0.124 $\pm$ 0.003 & 0.388 $\pm$ 0.003 & 6 \\
200   & 0.215 $\pm$ 0.010 & 0.589 $\pm$ 0.007 & 6 \\
300   & 0.316 $\pm$ 0.012 & 0.718 $\pm$ 0.002 & 6 \\
600   & 0.561 $\pm$ 0.016 & 0.979 $\pm$ 0.006 & 6 \\
1200  & 0.918 $\pm$ 0.051 & 1.792 $\pm$ 0.041 & 4 \\
\bottomrule
\end{tabular}
\caption{Average runtime (hours) by method and budget per prompt (mean $\pm$ SD).}
\label{tab:runtime}
\end{table}

\paragraph{Model Architecture and Training.} We employ the ModernBERT-base \citep{warner2024smarter} model. We use Adam optimizer~\citep{adam} to minimize the loss function in~\eqref{eq:agg_bce}, with a learning rate of 3e-5, and a batch size of $1500$. We train the model for 10 epochs and use early stopping based on the validation set. 




\subsection{Shared Implementation Details for Both Real and Synthetic Experiments}
\label{sec:training}

Our calibration framework applies to any machine learning model that outputs an estimated conditional quantile of the time-to-unsafe-sampling, denoted by \( \hat{q}_\tau(x) \), given the covariates \( x \). This is motivated by the fact that \( T \), the time-to-unsafe-sampling, is a Geometric random variable with parameter \( p(x) \), the probability of unsafe generation. Consequently, the conditional quantile function of \( T \) is given analytically in terms of \( p(x) \), and vice versa:
\[
q_\tau(x) = \left\lceil \frac{\log(1 - \tau)}{\log(1 - p(x))} \right\rceil
\quad \text{and} \quad
p(x) = 1 - (1 - \tau)^{1/q_\tau(x)}.
\]
The above relationship allows us to estimate either \( p(x) \) or \( q_\tau(x) \) and recover the other via a closed-form transformation.

\paragraph{Loss Function.} In our experiments, we estimate \( p(x) \) by fitting a model that minimizes the binary cross-entropy (BCE) loss over aggregate unsafe proportions. Specifically, for each prompt \( X_i \), we define the empirical success rate as \( \bar{Y}_i = \frac{1}{N} \sum_{j=1}^N Y_i^j \), where \( Y_i^j \in \{0,1\} \) indicates whether the \( j \)-th sample is unsafe. We then use the BCE loss:
\begin{equation}
\label{eq:agg_bce}
\mathrm{BCE}(\bar Y_i, \hat{p}(X_i)) = -\bigl[\bar{Y}_i \log \hat{p}(X_i) + (1 - \bar{Y}_i) \log(1 - \hat{p}(X_i))\bigr],
\end{equation}
where \( \hat{p}(X_i) \) is the model's prediction of the probability of unsafe generation. Notably, this loss is equivalent to the mean of the standard BCE loss across individual samples:
\[
\frac{1}{N}\sum_{j=1}^N \mathrm{BCE}(Y_i^j, \hat{p}(X_i)) = \mathrm{BCE}(\bar Y_i, \hat{p}(X_i)).
\]
This aggregation allows for a more computationally and memory-efficient implementation compared to the use of the vanilla BCE.

\subsubsection{Implementation Details of the \texttt{Naive} Calibration Procedure}
\label{sec:calib_impl}

In theory, the \ttnaive calibration (Algorithm~\ref{alg:calibration_naive}) can assign an unbounded censorship time per prompt, which leads to arbitrarily long running times. To avoid this, we restrict our search for the threshold \(\tau\) to the compact set \(\mathcal{T}\cap\tau_{\mathrm{prior}}\) (see Section~\ref{subsec:real}). This restriction in search space does not violate the method's validity, as it may only make it more conservative, as discussed in Section~\ref{subsec:prompt_adaptive}. This section also shows that assigning a censorship time $C_i \geq \hat{q}_\tau(X_i)$ is equivalent to drawing exactly $\hat{q}_\tau(X_i)$ generations. Conversely, setting $C_i < \hat{q}_\tau(X_i)$ amounts to not drawing any samples. Based on this observation, we implement a run-time efficient version of the \texttt{Naive} method that draws censorship times according to
\begin{equation*}
    C_i := \text{Ber}(g(X_i)) \cdot \hat{q}_{ \tau_{\text{prior}}}(X_i),
\end{equation*}
where $g(x)=\mathbb{P}(\mathrm{Geom}(\min(|\Itwo|/B, 1)\geq \hat{q}_{ \tau_{\text{prior}}}(X_i))$. It is important to note at this point that the average number of generations per prompt reported in Figure~\ref{fig:budget_effect} is the simulated number of samples that the equivalent, non-optimized \ttnaive solution would use. This implementation also implies that the \ttnaive method has a much smaller overall number of generations, leading to a faster runtime.

\subsection{Machine’s Spec}\label{sec:machine_spec}

The computational infrastructure used in the experiments are:
\begin{itemize}
    \item \textbf{CPU}: Intel(R) Xeon(R) CPU E5-2683 v4 @ 2.10GHz, Intel(R) Xeon(R) Gold 5318Y CPU @ 2.10GHz, Intel(R) Xeon(R) Gold 6336Y CPU @ 2.40GHz.
    \item \textbf{GPU}: NVIDIA A40, NVIDIA TITAN X (Pascal), NVIDIA 2080 TI, NVIDIA RTX 2060 SUPER.
    \item \textbf{OS}: Ubuntu 20.04.6.
\end{itemize}

\end{document}